\pdfoutput=1
\documentclass{article}
\usepackage{times,authblk}

\usepackage{geometry}
\usepackage{rotating, color,subfigure,algorithm,algorithmic}

\usepackage{multirow}
\usepackage{graphicx}
\usepackage{hyperref}                                                                                                                                                                                                       
\usepackage{enumerate}
\usepackage{amsthm,amsmath,amssymb}

\usepackage{mkolar_definitions}

\usepackage{ifthen}

\allowdisplaybreaks

\newenvironment{packed_enum}{
\begin{enumerate}
\setlength{\itemsep}{1pt}
\setlength{\parskip}{0pt}
\setlength{\parsep}{0pt}
}{\end{enumerate}}

\newcommand{\version}{arxiv}

\begin{document}
\title{On Estimating $L_2^2$ Divergence}

\author[1]{
Akshay Krishnamurthy
\thanks{akshaykr@cs.cmu.edu}}

\author[2]{
Kirthevasan Kandasamy
\thanks{kandasamy@cs.cmu.edu}}

\author[2]{
\\Barnab\'{a}s P\'{o}czos
\thanks{bapoczos@cs.cmu.edu}}

\author[3]{
Larry Wasserman
\thanks{larry@stat.cmu.edu}}

\affil[1]{Computer Science Department\\
Carnegie Mellon University}
\affil[2]{Machine Learning Department\\
Carnegie Mellon University}
\affil[3]{Statistics Department\\
Carnegie Mellon University}

\maketitle

\begin{abstract}
We give a comprehensive theoretical characterization of a nonparametric estimator for the $L_2^2$ divergence between two continuous distributions. 
We first bound the rate of convergence of our estimator, showing that it is $\sqrt{n}$-consistent provided the densities are sufficiently smooth.
In this smooth regime, we then show that our estimator is asymptotically normal, construct asymptotic confidence intervals, and establish a Berry-Ess\'{e}en style inequality characterizing the rate of convergence to normality.
We also show that this estimator is minimax optimal.

\end{abstract}

\ifthenelse{\equal{\version}{arxiv}}{\section{Introduction}}{\section{INTRODUCTION}}
\label{sec:intro}

One of the most natural ways to quantify the dis-similarity between two continuous distributions is with the $L_2$-distance between their densities. 
This distance -- which we typically call a divergence -- allows us to translate intuition from Euclidean geometry and consequently makes the $L_2$-divergence particularly interpretable.
Despite this appeal, we know of very few methods for estimating the $L_2$-divergence from data. 
For the estimators that do exist, we have only a limited understanding of their properties, which limits their applicability.
This paper addresses this lack of understanding with a comprehensive theoretical study of an estimator for the $L_2^2$-divergence.

Our estimator is the same kernel multi-sample $U$-statistic that has appeared numerous times in the literature~\cite{anderson1994two,gine2008simple}, but has, until now, lacked a complete theoretical development. 
Under a standard smoothness assumption, parameterized by $\beta$, and given $n$ samples from two densities supported over $\RR^d$, we establish the following properties.
\ifthenelse{\equal{\version}{arxiv}}{\begin{packed_enum}}
{\begin{enumerate}}
\item We analyze the rate of convergence in squared error, showing an $n^{\frac{-8\beta}{4\beta+d}}$ rate if $\beta < d/4$ and the parametric $n^{-1}$ rate if $\beta \ge d/4$ (Theorem~\ref{thm:rate}).
\item When $\beta > d/4$, we prove that the estimator is asymptotically normal (Theorem~\ref{thm:normality}). 
\item 
We derive a principled method for constructing a confidence interval that we justify with asymptotic arguments (Theorem~\ref{thm:asymp_conf}).
\item We also prove a Berry-Ess\'{e}en style inequality in the $\beta > d/4$ regime, characterizing the distance of the appropriately normalized estimator to the $\Ncal(0,1)$ limit (Theorem~\ref{thm:berry}). 
\item Lastly, we modify an existing proof to establish a matching lower bound on the rate of convergence (Theorem~\ref{thm:l2_lower}). 
  This shows that our estimator achieves the minimax rate.
\ifthenelse{\equal{\version}{arxiv}}{\end{packed_enum}}
{\end{enumerate}}
We are not aware of such a characterization of an estimator for this divergence.
Indeed, we are not aware of such a precise characterization for \emph{any} nonparametric divergence estimators.

The most novel technical ingredient of our work is the proof of Theorem~\ref{thm:berry}, where we upper bound the distance to the $\Ncal(0,1)$ limit of our estimator.
The challenges in this upper bound involve carefully controlling the bias in both our estimator and our estimator for its asymptotic variance so that we can appeal to classical Berry-Ess\'{e}en bounds. 
This technical obstacle arises in many nonparametric settings, but we are not aware of any related results. 

The remainder of this paper is organized as follows.
After mentioning some related ideas in Section~\ref{sec:related}, we specify the estimator of interest in Section~\ref{sec:estimator}.
In Section~\ref{sec:theory}, we present the main theoretical results, deferring proofs to Section~\ref{sec:proofs} and the appendix.
We conclude in Section~\ref{sec:discussion} with some future directions.

\ifthenelse{\equal{\version}{arxiv}}{\section{Related Work}}{\section{RELATED WORK}}
\label{sec:related}


There are a few other works that have considered estimation of the $L_2$-divergence under non-parametric assumptions~\cite{anderson1994two,poczos2011estimation,krishnamurthy2014nonparametric}.
Anderson et al. propose essentially the same estimator that we analyze in this paper~\cite{anderson1994two}.
When used for two-sample testing, they argue that one should not shrink the bandwidth with $n$, as it does not lend additional power to the test, while only increasing the variance. 
Unfortunately, this choice of bandwidth does not produce a consistent estimator.
When used for estimation, they remark that one should use a bandwidth that is smaller than for density estimation, but do not pursue this idea further.
By formalizing this undersmoothing argument, we achieve the parametric $n^{-1}$ squared error rate.



Poczos et al. establish consistency of a nearest neighbor based $L_2$ divergence estimator, but do not address the rate of convergence or other properties~\cite{poczos2011estimation}. 
Krishnamurthy et al. propose an estimator based on a truncated Fourier expansion of the densities~\cite{krishnamurthy2014nonparametric}.
They establish a rate of convergence that we match, but do not develop any additional properties.
Similarly, K\"{a}llberg and Seleznjev propose an estimator based on $\epsilon$-nearest neighbors and prove similar asymptotic results to ours, but they do not establish Berry-Esse\'{e}n or minimax lower bounds~\cite{kallberg2012estimation}.
In contrast to both of these works, our estimator and our analysis are considerably simpler, which facilitates both applicability and theoretical development.

As will become clear in the sequel, our estimator is closely related to the maximum mean discrepancy (MMD) for which we have a fairly deep understanding~\cite{gretton2012kernel}.
While the estimators are strikingly similar, they are motivated from vastly different lines of reasoning and the analysis reflects this difference. 
The most notable difference is that with MMD, the population quantity is \emph{defined} by the kernel and bandwidth. 
That is, the choice of kernel influences not only the estimator but also the population quantity.
We believe that our estimand is more interpretable as it is independent of the practioner's choices.
Nevertheless, some of our results, notably the Berry-Ess\'{e}en bound, can be ported to an estimate of the MMD, advancing our understanding of this quantity.

There is a growing body of literature on estimation of various divergences under nonparametric assumptions. 
This line of work has primarily focused on Kullback-Leibler, Renyi-$\alpha$, and Csiszar $f$-divergences~\cite{leonenko2008class,nguyen2010estimating,perez2008kullback}. 
As just one example, Nguyen et al. develop a convex program to estimate $f$-divergences under the assumption that the density ratio belongs to a reproducing kernel Hilbert space. 
Unfortunately, we have very little understanding as to which divergence is best suited to a particular problem, so it is important to have an array of estimators at our disposal.

Moreover, apart from a few examples, we do not have a complete understanding of the majority of these estimators.
In particular, except for the MMD~\cite{gretton2012kernel}, we are unaware of principled methods for building confidence intervals for any of these divergences, and this renders the theoretical results somewhat irrelevant for testing and other inference problems.


Our estimator is based on a line of work studying the estimation of integral functionals of a density in the nonparametric setting~\cite{gine2008simple,laurent1996efficient,kerkyacharian1996estimating,birge1995estimation,bickel1988estimating}.
These papers consider estimation of quantities of the form $\theta = \int f(p, p^{(1)}, \ldots, p^{(k)}) d\mu$, where $f$ is some known functional and $p^{(i)}$ is the $i$th derivative of the density $p$, given a sample from $p$.
Gin\'{e} and Nickl specifically study estimation of $\int p(x)^2d\mu$ and our work generalizes their results to the $L_2^2$-divergence functional~\cite{gine2008simple}. 

Turning to lower bounds, while we are not aware of a lower bound for $L_2^2$-divergence estimation under nonparametric assumptions, there are many closely related results.
For example, Birge and Massart~\cite{birge1995estimation} establish lower bounds on estimating integral functionals of a single density, while Krishnamurthy et al. extend their proof to a class of divergences~\cite{krishnamurthy2014nonparametric}.
Our lower bound is based on some modifications to the proof of Krishnamurthy et al.

\ifthenelse{\equal{\version}{arxiv}}{\section{The Estimator}}{\section{THE ESTIMATOR}}
\label{sec:estimator}
Let $\PP$ and $\QQ$ be two distributions supported over $\RR^d$ with Radon-Nikodym derivatives (densities) $p \triangleq d\PP/d\mu, q \triangleq d\QQ/d\mu$ with respect to a measure $\mu$.
The $L_2^2$ divergence between these two distributions, denoted throughout this paper as $D(p,q)$ is defined as:
\ifthenelse{\equal{\version}{arxiv}}{
\begin{align*}
D(p,q) \triangleq \int (p(x) - q(x))^2 d\mu(x) = \underbrace{\int p^2(x)d\mu}_{\theta_p} + \underbrace{\int q^2(x)d\mu}_{\theta_q} - 2 \underbrace{\int p(x)q(x)d\mu}_{\theta_{p,q}}.
\end{align*}
}
{
\begin{align*}
D(p,q) &\triangleq \int (p(x) - q(x))^2 d\mu(x) \\
&= \underbrace{\int p^2(x)d\mu}_{\theta_p} + \underbrace{\int q^2(x)d\mu}_{\theta_q} - 2 \underbrace{\int p(x)q(x)d\mu}_{\theta_{p,q}}.
\end{align*}
}
Estimation of the first two terms in the decomposition has been extensively studied in the nonparametric statistics community~\cite{laurent1996efficient,birge1995estimation,gine2008simple,bickel1988estimating}.
For these terms, we use the kernel-based U-statistic of Gine and Nickl~\cite{gine2008simple}.
For the bilinear term, $\theta_{p,q}$, we use a natural adaptation of their U-statistic to the multi-sample setting.
Specifically, given samples $\{X_i\}_{i=1}^{2n}\sim p, \{Y_i\}_{i=1}^{2n} \sim q$, we estimate $\theta_p$ with $\hat{\theta}_p$ and $\theta_{p,q}$ with $\hat{\theta}_{p,q}$, given by:
\ifthenelse{\equal{\version}{arxiv}}{
\begin{align}
\hat{\theta}_p = \frac{1}{n(n-1)}\sum_{i \ne j=1}^n \frac{1}{h^d}K\left(\frac{X_i-X_j}{h}\right)\qquad
\hat{\theta}_{p,q} = \frac{1}{n^2}\sum_{i,j=n+1}^{2n} \frac{1}{h^d}K\left(\frac{X_i - Y_j}{h}\right),
\end{align}
}
{
\begin{align}
\hat{\theta}_p &= \frac{1}{n(n-1)}\sum_{i \ne j=1}^n \frac{1}{h^d}K\left(\frac{X_i-X_j}{h}\right)\\
\hat{\theta}_{p,q} &= \frac{1}{n^2}\sum_{i,j=n+1}^{2n} \frac{1}{h^d}K\left(\frac{X_i - Y_j}{h}\right),
\end{align}
}
where $K: \RR^d \rightarrow \RR_{\ge 0}$ is a kernel function and $h \in \RR_{\ge 0}$ is a bandwidth parameter.
In Assumption~\ref{assump:main} below, we prescibe some standard restrictions on the kernel and a scaling of the bandwidth.

The squared term involving $q$, $\theta_q$, is estimated analogously to $\theta_p$, and we denote the estimator $\hat{\theta}_q$. 
The final $L_2^2$-divergence estimator is simply $\hat{D}(p,q) = \hat{\theta}_p + \hat{\theta}_q - 2 \hat{\theta}_{p,q}$. 
Notice that we have split the data so that each point $X_i$ (respectively $Y_j$) is used in exactly one term.
Forcing this independence will simplify our theoretical analysis without compromising the properties. 

While data-splitting facilitates our theoretical analysis, for some applications, we recommend against it as it does not make effective use of the sample.
It is straightforward to prove the same rate of convergence for the estimator without data splitting.
As a consequence, some applications, such as machine learning on distributions~\cite{poczos2012nonparametric}, may not require splitting the sample.
However, asymptotic normality, the confidence interval and the Berry-Ess\'{e}en bound do rely crucially on the data-splitting technique, so it is necessary to split the sample for most inference problems.

In fact, without data-splitting, the limiting distribution is \emph{not} always normal.
When $p=q$, which is the relevant setting for two-sample testing, Gretton \emph{et al.} show that the limiting distribution for the $U$-statistic MMD estimator is a weighted sum of products of gaussian random variables~\cite{gretton2012kernel}.
Essentially the same argument applies here, showing that data-splitting is critical for our asymptotic results.

We also remark that the estimator can naively be computed in quadratic time. 
However, with a compact kernel, a number of data structures are available that lead to more efficient implementations.
In particular, the dual tree algorithm of Ram et al. can be used to compute $\hat{D}$ in linear time~\cite{ram2009linear}.

\ifthenelse{\equal{\version}{arxiv}}{\section{Theoretical Properties}}{\section{THEORETICAL PROPERTIES}}
\label{sec:theory}
In this section, we highlight some of the theoretical properties enjoyed by the divergence estimator $\hat{D}$.
We begin with stating the main assumptions, regarding the smoothness of the densities, properties of the kernel, and the choice of bandwidth $h$.

\begin{definition}
We call $\Wcal_1^\beta(C)$, for $\beta \in \NN$ and $C > 0$, the \textbf{Bounded Variation class} of order $\beta$ which is the set of $\beta$-times differentiable funtions whose $\beta$th derivatives have bounded $L_1$ norm.
Formally, a function $f: \RR^d \rightarrow \RR$ belongs to $\Wcal_1^\beta(C)$ if for all tuples of natural numbers $r=(r_1, \ldots, r_d)$ with $\sum_j r_j \le \beta$ we have $\|D^r f\|_1 \le C$, where $D^r = \frac{\partial^{r_1+\ldots+r_d}}{\partial x_1^{r_1} \ldots \partial x_d^{r_d}}$ is a derivative operator. 
\end{definition}

\begin{assum} Assume $p,q,K$, and $h$ satisfy:
\begin{packed_enum}
\item \textbf{Smoothness}: The densities $p,q$ belong to the bounded variation class $\Wcal_1^{\beta}(C)$.
\item \textbf{Kernel Properties}: $K$ is bounded, symmetric, supported on $(-1,1)^d$, and has $\int K(u)d\mu(u) = 1$. 
$\int \prod_ix_i^{r_i} K(x)dx = 0$ for all $(r_1, \ldots, r_d)$ with $\sum_{j} r_j \le 2\beta$. 
\item \textbf{Kernel Bandwidth}: We choose $h \asymp n^{\frac{-2}{4\beta+d}}$. 
\end{packed_enum}
\label{assump:main}
\end{assum}

The smoothness assumption is similar in spirit to both the H\"{o}lder and Sobolev assumptions which are more standard in the nonparametric literature. 
Specifically, the bounded variation assumption is the integrated analog of the H\"{o}lder assumption, which is a pointwise characterization of the function.
It is also the $L_1$ analog of the Sobolev assumption, which requires that $||D^r f||_2^2$ is bounded. 

One difference is that the class $\Wcal_1^\beta$ can not be defined for non-integral smoothness, $\beta$, while both the H\"{o}lder and Sobolev classes can. 
While our results can be shown for the Sobolev class, working with bounded variation class considerably simplifies the proofs as we avoid the need for any Fourier analysis. 
The H\"{o}lderian assumption is insufficient as H\"{o}lder smoothness is not additive under convolution, which is critical for establishing the low order bias of our estimator.

The kernel properties are now fairly standard in the literature.
Notice that we require the kernel to be of order $2\beta$, instead of order $\beta$ as is required in density estimation. 
This will allow us to exploit additional smoothness provided by the convolution implicit in our estimators. 
Of course one can construct such kernels for any $\beta$ using the Legendre polynomials~\cite{tsybakov2009introduction}.
We remark that the scaling of the kernel bandwidth is not the usual scaling used in density estimation. 

We now turn to characterizing the rate of convergence of the estimator $\hat{D}$.
While we build off of the analysis of Gin\'{e} and Nickl, who analyze the estimator $\hat{\theta}_p$~\cite{gine2008simple}, our proof has two main differences.
First, since we work with a different smoothness assumption, we use a different technique to control the bias.
Second, we generalize to the bilinear term $\hat{\theta}_{p,q}$, which involves some modifications.
We have the following theorem:
\begin{theorem}
\label{thm:rate}
Under Assumption~\ref{assump:main} we have:
\begin{align}
\EE[(\hat{D}(p,q) - D(p,q))^2] \le \left\{\begin{aligned} c_3n^{\frac{-8\beta}{4\beta+d}} & \textrm{ if } \beta < d/4\\ c_4 n^{-1} & \textrm{ if } \beta \ge d/4 \end{aligned}\right.
\end{align}
\end{theorem}

Notice that the rate of convergence is substantially faster than the rate of convergence for estimation of $\beta$-smooth densities.
In particular, the parametric rate is achievable provided sufficient smoothness\footnote{The parametric rate is $n^{-1}$ in squared error which implies an $n^{-1/2}$ rate in absolute error.}.
This agrees with the results on estimation of integral functionals in the statistics community~\cite{gine2008simple,birge1995estimation}.
It also matches the rate of the orthogonal series estimator studied by Krishnamurthy et al.~\cite{krishnamurthy2014nonparametric}.


One takeaway from the theorem is that the one should \emph{not} use the optimal density estimation bandwidth of $n^{\frac{-1}{2\beta+d}}$ for this problem. 
As we mentioned, this choice was analyzed by Anderson et al. and results in a slower convergence rate~\cite{anderson1994two}. 
Indeed our choice of bandwidth $h \asymp n^{\frac{-2}{4\beta+d}}$ is always smaller, so we are undersmoothing the density estimate. 
The additional variance induced by undersmoothing is mitigated by integration in the estimand, leading to a faster rate of convergence. 

Interestingly, there seem to be two distinct approaches to estimating integral functionals.
On one hand, one could plug in an undersmoothed density estimator directly into the functional. 
This is the approach we take here and it has also been used for other divergence estimation problems~\cite{poczos2011estimation}.
Another approach is to plug in a minimax optimal density estimator and then apply some post-hoc correction.
This latter approach can be shown to achieve similar rates for divergence estimation problems~\cite{krishnamurthy2014nonparametric}.
Note that our method can be computationally much simpler. 

The next theorem establishes asymptotic normality in the smooth regime:
\begin{theorem}
\label{thm:normality}
When $\beta > d/4$:
\[
\sqrt{n}\left(\hat{D}(p,q) - D(p,q)\right) \leadsto \Ncal(0, \sigma^2),
\]
where $\leadsto$ denotes convergence in distribution and:
\ifthenelse{\equal{\version}{arxiv}}{
\begin{align}
\sigma^2 = 4\Var_{x \sim p}(p(x)) + 4\Var_{y \sim q}(q(y)) + 4\Var_{x \sim p}(q(x)) + 4\Var_{y \sim q}(p(y)
\label{eq:asymp_var}
\end{align}
}
{
\begin{align}
\sigma^2 = \left\{ \begin{aligned} & 4\Var_{x \sim p}(p(x)) + 4\Var_{y \sim q}(q(y)) \\
& + 4\Var_{x \sim p}(q(x)) + 4\Var_{y \sim q}(p(y)\end{aligned}\right. 
\label{eq:asymp_var}
\end{align}
}
\end{theorem}

With this characterization of the limiting distribution, we can now turn to construction of an asymptotic confidence interval. 

The most straightforward approach is to estimate the asymptotic variance and appeal to Slutsky's Theorem.
We simply use a plugin estimator for the variance, which amounts to replacing all instances of $p,q$ in Equation~\ref{eq:asymp_var} with estimates $\hat{p}, \hat{q}$ of the densities.
For example, we replace the first term with $\int \hat{p}(x)^3 - (\int \hat{p}(x)^2)^2$. 
We denote the resulting estimator by $\hat{\sigma}^2$, and mention that one should use a bandwidth $h \asymp n^{\frac{-1}{2\beta+d}}$ for estimating this quantity.

In Section~\ref{sec:proofs} (specifically Lemma~\ref{lem:var_est}), we bound the rate of convergence of this estimator, and its consistency immediately gives an asymptotic confidence interval:
\begin{theorem}
Let $z_{\alpha/2} = \Phi^{-1}(1 - \alpha/2)$ be the $1-\alpha/2$th quantile of the standard normal distribution.
Then,
\begin{align}
\frac{\sqrt{n}(\hat{D}(p,q) - D(p,q))}{\hat{\sigma}} \leadsto \Ncal(0,1),
\end{align}
whenever $\beta > d/4$. Consequently,
\begin{align}
\PP\left( D \in \left[\hat{D} - \frac{z_{\alpha/2}\hat{\sigma}}{\sqrt{n}}, \hat{D} + \frac{z_{\alpha/2}\hat{\sigma}}{\sqrt{n}}\right]\right) \rightarrow 1-\alpha
\end{align}
which means that $[\hat{D} - \frac{z_{\alpha/2}\hat{\sigma}}{\sqrt{n}}, \hat{D} + \frac{z_{\alpha/2}\hat{\sigma}}{\sqrt{n}}]$ is an asymptotic $1-\alpha$ confidence interval for $D$.
\label{thm:asymp_conf}
\end{theorem}

While the theorem does lead to a confidence interval, it is worth asking how quickly the distribution of the self-normalizing estimator converges to a standard normal, so that one has a sense for the quality of the interval in finite sample.
We therefore turn to establishing a more precise guarantee. 
To simplify the presentation, we assume that we have a fresh set of $n$ samples per distribution to compute $\hat{\sigma}^2$. 
Thus we are given $3n$ samples per distribution in total, and we use $2n$ of them to compute $\hat{D}$ and the last set for $\hat{\sigma}^2$. 
As before, in computing $\hat{\sigma}^2$, we set $h \asymp n^{\frac{-1}{2\beta+d}}$.

\begin{theorem}
Let $\Phi(z)$ denote the CDF of the standard normal. 
Under Assumption~\ref{assump:main}, there exists a constant $c_\star > 0$ such that:
\ifthenelse{\equal{\version}{arxiv}}{
\begin{align}
\sup_z \left| \PP\left( \frac{\sqrt{n}(\hat{D}(p,q) - D(p,q))}{\hat{\sigma}} \le z\right) - \Phi(z)\right| \le 
c_\star\left(n^{\frac{d-4\beta}{8\beta+d}} + n^{\frac{-\beta/2}{2\beta+d}}\right).
\end{align}
}{
\begin{align}
\sup_z \left| \PP\left( \frac{\sqrt{n}(\hat{D}(p,q) - D(p,q))}{\hat{\sigma}} \le z\right) - \Phi(z)\right| \le \\
c_\star\left(n^{\frac{d-4\beta}{8\beta+d}} + n^{\frac{-\beta/2}{2\beta+d}}\right).
\end{align}
}
This bound is $o(1)$ as soon as $\beta > d/4$.
\label{thm:berry}
\end{theorem}

As an immediate consequence of the theorem, we obtain an error bound on the quality of approximation of the confidence interval in Theorem~\ref{thm:asymp_conf}.
We remark that one can explicitly track all of the constants in the theorem and leave the result in terms of the bandwidth $h$ and problem dependent constants, although this is somewhat tedious.
For ease of exposition we have chosen to present the asymptotic version of the theorem, focusing instead on the rate of convergence to the limiting $\Ncal(0,1)$ distribution.

It is not surprising that the rate of convergence to Gaussianity is not the typical $n^{-1/2}$ rate, as it depends on the third moment of the $U$-statistic, which is decreasing with $n$.
It also depends on the non-negligible bias of the estimator.
However, as soon as $\beta > d/4$, it is easily verified that the bound is $o(1)$.
This matches our asymptotic guarantee in Theorem~\ref{thm:normality}.
Of course, for smoother densities, the rate of convergence in the theorem is polynomially faster.

In addition to the practical consequences, we believe the techniques used in the proof of the theorem are fairly novel.
While establishing Berry-Ess\'{e}en bounds for linear and other parameteric estimators is fairly straightforward~\cite{chen2010normal}, this type of result is uncommon in the nonparametric literature.
The main challenge is dealing with the bias and additional error introduced by estimating the variance. 

Finally, let us address the question of optimality. 
The following theorem lower bounds the rate of convergence of any estimator for the $L_2^2$ divergence, when the densities belong to the bounded variation class.
\begin{theorem}
\label{thm:l2_lower}
With $\gamma_\star = \min\{8\beta/(4\beta+d), 1\}$ and for any $\epsilon > 0$, we have:
\begin{align}
\inf_{\hat{D}_n} \sup_{p,q \in \Wcal_1^\beta(C)} \PP_{p,q}^n \left[ (\hat{D}_n - D)^2 \ge \epsilon n^{-\gamma_\star}\right] \ge c > 0
\end{align}
\end{theorem}
The result shows that $n^{-\gamma_\star}$ lower bounds the minimax rate of convergence in squared error.
Of course $\gamma_\star = 1$ when $\beta \ge d/4$, so the rate of convergence can be no better than the parametric rate.
Comparing with Theorem~\ref{thm:rate}, we see that our estimator achieves the minimax rate.

\ifthenelse{\equal{\version}{arxiv}}{\section{Proofs}}{\section{PROOFS}}
\label{sec:proofs}

The proofs of Theorems~\ref{thm:rate} and~\ref{thm:normality} are based on modifications to the analysis of Gin\'{e} and Nickl~\cite{gine2008simple} so we will only sketch the ideas here.
The majority of this section is devoted to proving the Berry-Ess\'{e}en bound in Theorem~\ref{thm:berry}, proving Theorem~\ref{thm:asymp_conf} along the way.
We close the section with a sketch of the proof of Theorem~\ref{thm:l2_lower}.

\subsection{Proof Sketch of Theorem~\ref{thm:rate} and~\ref{thm:normality}}
Theorem~\ref{thm:rate} follows from bounding the bias and the variance of the terms $\hat{\theta}_p, \hat{\theta}_q$, and $\hat{\theta}_{pq}$. 
The terms are quite similar and we demonstrate the ideas with $\hat{\theta}_{pq}$. 

We show that the bias can be written in terms of a convolution and then use the fact that bounded-variation smoothness is additive under convolution.
By a substitution, we see that the bias for $\hat{\theta}_{pq}$ is:
\ifthenelse{\equal{\version}{arxiv}}{
\begin{align*}
\EE[\hat{\theta}_{pq}] - \theta_{pq} = \int \int K(u)[p(x - uh) - p(x)]q(x)dudx
= \int K(u)[(p_0 \star q)(uh) - (p_0 \star q)(0)] du,
\end{align*}
}{
\begin{align*}
\EE[\hat{\theta}_{pq}] - \theta_{pq} &= \int \int K(u)[p(x - uh) - p(x)]q(x)dudx \\
&= \int K(u)[(p_0 \star q)(uh) - (p_0 \star q)(0)] du,
\end{align*}
}
where $p_0(x) = p(-x)$ and $\star$ denotes convolution.
Next, we use Young's inequality to show that if two functions $f,g$ belong to $\Wcal_1^{\beta}(C)$, then $f\star g \in \Wcal_1^{2\beta}(C^2)$. 
Using this inequality, we can take a Taylor expansion of order $2\beta-1$ and use the kernel properties to annihilate all but the remainder term, which is of order $h^{2\beta}$.

To bound the variance, we expand:
\begin{align*}
\EE[\hat{\theta}_p^2] = \EE\left[\frac{1}{n^2(n-1)^2}\sum_{i \ne j, s \ne t} K_h(X_i, Y_j), K_h(X_s, Y_t)\right]
\end{align*}
By analyzing each of the different scenaries (i.e. the terms where all indices are different, there is one equality, or there is two equalities), it is not hard to show that the variance is:
\begin{align*}
\Var(\hat{\theta}_p) \le O\left(\frac{1}{n} + \frac{1}{h^dn^2}\right)
\end{align*}
Equipped with these bounds, the rate of convergence follows from the bias-variance decomposition, and our choice of bandwidth.

The proof of normality is quite technical and we just briefly comment on the steps, deferring all calculations to the appendix.
We apply Hoeffding's decomposition, writing the centered estimator as the sum of a $U$-process and two empirical processes, one for $p$ and one for $q$.
The $U$-process converges in quadratic mean to $0$ at faster than $1/\sqrt{n}$ rate, so it can be ignored.
For the empirical processes, we show that they are close (in quadratic mean) to $\sqrt{n}(P_n q - \theta_{pq})$ and $\sqrt{n}(Q_np - \theta_{pq}$, where $P_n, Q_n$ are the empirical measures.
From here, we apply the Lindberg-Levy central limit theorem to these empirical processes. 


\subsection{Proof of Theorem~\ref{thm:berry}}

The Berry-Ess\'{e}en theorem can be applied to an unbiased multi-sample $U$-statistic, normalized by a term involving the conditional variances.
Specifically, we will be able to apply the theorem to:
\begin{align}
\label{eq:berry_esseen_quantity}
\frac{\sqrt{n}(\hat{D} - \EE\hat{D})}{\bar{\sigma}},
\end{align}
where:
\ifthenelse{\equal{\version}{arxiv}}{
\begin{align*}
\bar{\sigma}^2 = 4 \Var_{x \sim p}(\bar{p}(x)) + 4 \Var_{y \sim q}(\bar{q}(y))
  + 4 \Var_{x \sim p}(\bar{q}(x)) + 4 \Var_{y \sim q}(\bar{p}(y))
\end{align*}
}{
\begin{align*}
\bar{\sigma}^2 = \left\{ \begin{aligned} &4 \Var_{x \sim p}(\bar{p}(x)) + 4 \Var_{y \sim q}(\bar{q}(y)) \\ 
  & + 4 \Var_{x \sim p}(\bar{q}(x)) + 4 \Var_{y \sim q}(\bar{p}(y))\end{aligned}\right.
\end{align*}
}
The appropriate normalization is similar to the asymptotic variance $\sigma^2$ (Equation~\ref{eq:asymp_var}) except that the densities are replaced with the mean of their kernel density estimates, i.e. $\bar{p}(x) = \int K_h(x,y)p(y)$.

We would like to establish a Berry-Ess\'{e}en bound for $\sqrt{n}\hat{\sigma}^{-1}(\hat{D} - D)$, but must first make several translations to arrive at Equation~\ref{eq:berry_esseen_quantity}.
We achieve this with several applications of the triangle inequality and some Gaussian anti-concentration properties.
We must also analyze the rate of convergence of the variance estimator $\hat{\sigma}^2$ to $\bar{\sigma}^2$ for this bound and to $\sigma^2$ for Theorem~\ref{thm:asymp_conf}.

Let $F_{\hat{\sigma}}$ be the distribution of $\hat{\sigma}/\bar{\sigma}$, induced by the second half of the sample.
Then we may write:
\ifthenelse{\equal{\version}{arxiv}}{
\begin{align*}
\PP\left(\frac{\sqrt{n}}{\hat{\sigma}}(\hat{D} - D) \le z\right)
= \int \PP\left(\frac{\sqrt{n}}{\bar{\sigma}}(\hat{D} - D) \le tz \right) dF_{\hat{\sigma}}(t),
\end{align*}
}{
\begin{align*}
&\PP\left(\frac{\sqrt{n}}{\hat{\sigma}}(\hat{D} - D) \le z\right) \\
&= \int \PP\left(\frac{\sqrt{n}}{\bar{\sigma}}(\hat{D} - D) \le tz \right) dF_{\hat{\sigma}}(t),
\end{align*}
}
so that we can decompose the proximity to the standard normal CDF as:
\ifthenelse{\equal{\version}{arxiv}}{
\begin{align*}
&\sup_z \left|\PP\left(\frac{\sqrt{n}}{\hat{\sigma}}(\hat{D} - D) \le z\right)- \Phi(z)\right| \le \\
&\sup_z \int \left| \PP\left(\frac{\sqrt{n}}{\bar{\sigma}}(\hat{D} - D) \le tz\right) - \Phi(tz)\right| dF_{\hat{\sigma}}(t)
+ \sup_z \left| \int \Phi(tz)dF_{\hat{\sigma}}(t) - \Phi(z)\right|.
\end{align*}
}{
\begin{align*}
& \sup_z \left|\PP\left(\frac{\sqrt{n}}{\hat{\sigma}}(\hat{D} - D) \le z\right)- \Phi(z)\right| \\
& \le \sup_z \int \left| \PP\left(\frac{\sqrt{n}}{\bar{\sigma}}(\hat{D} - D) \le tz\right) - \Phi(tz)\right| dF_{\hat{\sigma}}(t)\\
& + \sup_z \left| \int \Phi(tz)dF_{\hat{\sigma}}(t) - \Phi(z)\right|.
\end{align*}
}
For the first term it is quite easy to eliminate the integral by pushing the supremum inside and replacing $tz$ with the variable being maximized.
This leads to:
\ifthenelse{\equal{\version}{arxiv}}{
\begin{align*}
& \sup_z \left| \PP\left(\frac{\sqrt{n}}{\bar{\sigma}}(\hat{D} - D) \le z\right) - \Phi(z) \right| \le \\
& \sup_z \left| \PP\left(\frac{\sqrt{n}}{\bar{\sigma}}(\hat{D} - \EE\hat{D}) \le z\right) - \Phi\left(z\right)\right|
+ \sup_z \left|\Phi\left(z - \frac{\sqrt{n}}{\bar{\sigma}}(\EE\hat{D} - D)\right) - \Phi(z)\right|,
\end{align*}
}{
\begin{align*}
& \sup_z \left| \PP\left(\frac{\sqrt{n}}{\bar{\sigma}}(\hat{D} - D) \le z\right) - \Phi(z) \right| \\
& \le \sup_z \left| \PP\left(\frac{\sqrt{n}}{\bar{\sigma}}(\hat{D} - \EE\hat{D}) \le z\right) - \Phi\left(z\right)\right|\\
& + \sup_z \left|\Phi\left(z - \frac{\sqrt{n}}{\bar{\sigma}}(\EE\hat{D} - D)\right) - \Phi(z)\right|,
\end{align*}
}
which follows by adding and subtracting $\EE\hat{D}$, adding and subtracting a term involving the Gaussian CDF and the bias and redefining $z$ in the first term. 
The first term on the right hand side involves the expression in Equation~\ref{eq:berry_esseen_quantity} and we will apply Theorem 10.4 from Chen et al. to control it~\cite{chen2010normal}.
The second term can be bounded since $\EE\hat{D} - D \asymp h^{2\beta}, \sigma = \Theta(1)$ and the Gaussian density is at most $(2\pi)^{-1/2}$. 
This gives:
\begin{align}
\sup_z \left|\Phi\left(z - \frac{\sqrt{n}}{\sigma}(\EE\hat{D} - D)\right) - \Phi(z)\right| \le c_b \sqrt{n}h^{2\beta}.
\label{eq:bias_cdf}
\end{align}

Returning to the term involving the variance estimator, we will need the following lemma, which bounds the error in the variance estimate:
\begin{lemma}
\label{lem:var_est}
Under Assumption~\ref{assump:main}, but with $h \asymp n^{\frac{-1}{2\beta+d}}$, we have that for any $\epsilon > 0$:
\ifthenelse{\equal{\version}{arxiv}}{
\begin{align}
\PP[|\hat{\sigma}^2 - \sigma^2| > \epsilon] \le C_1 \epsilon^{-1}n^{\frac{-\beta}{2\beta+d}}, \qquad
\PP[|\hat{\sigma}^2 - \bar{\sigma}^2| > \epsilon] \le C_2 \epsilon^{-1}n^{\frac{-\beta}{2\beta+d}}.
\end{align}
}
{
\begin{align}
\PP[|\hat{\sigma}^2 - \sigma^2| > \epsilon] &\le C_1 \epsilon^{-1}n^{\frac{-\beta}{2\beta+d}},\\
\PP[|\hat{\sigma}^2 - \bar{\sigma}^2| > \epsilon] &\le C_2 \epsilon^{-1}n^{\frac{-\beta}{2\beta+d}}.
\end{align}
}
\end{lemma}
The first part of Lemma~\ref{lem:var_est} immediately gives the asymptotic confidence interval in Theorem~\ref{thm:asymp_conf}, as we have a consistent estimator of the asymptotic variance. 
The second part is used in the Berry-Ess\'{e}en bound. 

Notice that since $\bar{\sigma}, \bar{\sigma}^2 = \Theta(1)$ and since $\hat{\sigma}^2 > 0$, we also have that:
\begin{align*}
\PP[|\hat{\sigma} - \bar{\sigma}| > \epsilon] \le C \epsilon^{-1} n^{\frac{-\beta}{2\beta+d}},
\end{align*}
where the constant has changed slightly.
Since $F_{\hat{\sigma}}$ is the CDF for $\hat{\sigma}/\sigma$ and since the difference between two Gaussian CDFs is bounded by two, we therefore have,
\ifthenelse{\equal{\version}{arxiv}}{
\begin{align*}
\int_{-\infty}^{1-\epsilon}\Phi(tz) - \Phi(z)dF_{\hat{\sigma}}(t) + \int_{1+\epsilon}^{\infty}\Phi(tz) - \Phi(z)dF_{\hat{\sigma}}(t)
 \le C \epsilon^{-1}n^{\frac{-\beta}{2\beta+d}}.
\end{align*}
}{
\begin{align*}
\int_{-\infty}^{1-\epsilon}\Phi(tz) - \Phi(z)dF_{\hat{\sigma}}(t) + \int_{1+\epsilon}^{\infty}\Phi(tz) - \Phi(z)dF_{\hat{\sigma}}(t)\\
 \le C \epsilon^{-1}n^{\frac{-\beta}{2\beta+d}}.
\end{align*}
}
So we only have to consider the situation where $1-\epsilon \le t \le 1+\epsilon$. 
The difference between the Gaussian CDF at $z$ and $(1-\epsilon)z$ is small, since while the width of integration is growing linearly, the height of the integral is decaying exponentially. 
This term is maximized at $\pm 1$ and it is $O(\epsilon)$, so that the entire term depending on the variance estimate is:
\begin{align}
\left| \int \Phi(tz)dF_{\hat{\sigma}}(t) - \Phi(z)\right| \le O\left(\epsilon + n^{\frac{-\beta}{2\beta+d}}/\epsilon\right).
\label{eq:variance_cdf}
\end{align}
Optimizing over $\epsilon$ gives a rate of $O(n^{\frac{-\beta/2}{2\beta+d}})$. 

The Berry-Ess\'{e}en inequality applied to the term $\sqrt{n}\sigma^{-1}(\hat{D} - \EE\hat{D})$ reveals that:
\ifthenelse{\equal{\version}{arxiv}}{
\begin{align}
\sup_z \left| \PP\left(\frac{\sqrt{n}}{\sigma}(\hat{D} - \EE\hat{D}) \le z\right) - \Phi\left(z\right)\right|
\le O\left(n^{-1/2} + \frac{1}{\sqrt{nh^d}}\right),
\label{eq:berry_app}
\end{align}
}{
\begin{align}
& \sup_z \left| \PP\left(\frac{\sqrt{n}}{\sigma}(\hat{D} - \EE\hat{D}) \le z\right) - \Phi\left(z\right)\right| \nonumber \\
& \le O\left(n^{-1/2} + \frac{1}{\sqrt{nh^d}}\right),
\label{eq:berry_app}
\end{align}
}
where all of the constants can be tracked explicitly, although they depend on the unknown densities $p,q$. 
The application of the theorem from Chen et al. requires bounding various quantities related to the moments of the $U$-statistic.
All of these terms can be bounded using straightforward techniques and we defer these details along with some more careful book-keeping to the appendix.

Theorem~\ref{thm:berry} follows from the application of Berry-Ess\'{e}en in Equation~\ref{eq:berry_app}, the variance bound in Equation~\ref{eq:variance_cdf}, the bias bound in Equation~\ref{eq:bias_cdf} and our choice of bandwidth in Assumption~\ref{assump:main}. 

\subsection{Proof of Theorem~\ref{thm:l2_lower}}
The proof is a modification of Theorem 2 of~\cite{krishnamurthy2014nonparametric}.
The idea is to reduce the estimation problem to a simple hypothesis test, and then lower bound the probability of error by appealing to the Neyman-Pearson Lemma.
If the null and alternative hypotheses, which will consist of pairs of distributions, are well separated, in the sense that the $L_2^2$ divergence of the null hypothesis is far from the divergence of the alternative, then a lower bound on the probability of error immediately lower bounds the estimation error. 
This argument is formalized in the following Lemma from~\cite{krishnamurthy2014nonparametric}, which is a consequence of Theorem 2.2 of Tsybakov~\cite{tsybakov2009introduction}.

\begin{lemma}[\cite{krishnamurthy2014nonparametric}]
\label{lem:lecam}
Let $\Lambda$ be an index set and let $p_0,q_0 p_\lambda \forall \lambda \in \Lambda$ be densities (with corresponding distribution functions $P_0,Q_0,P_\lambda$) belonging to a function space $\Theta$. 
Let $T$ be a bivariate functional defined on some subset of $\Theta \times \Theta$ which contains $(p_0,q_0)$ and $(p_\lambda, q_0) \forall \lambda \in \Lambda$. 
Define $\overline{P^n} = \frac{1}{|\Lambda|}\sum_{\lambda \in \Lambda} P_\lambda^n$.
If:
\ifthenelse{\equal{\version}{arxiv}}{
\begin{align*}
h^2(P_0^n \times Q_0^n, \overline{P^n} \times Q_0^n) \le \gamma < 2, \textrm{ and }
T(p_0,q_0) \ge 2\beta + T(p_\lambda, q) \forall \lambda \in \Lambda
\end{align*}
}{
\begin{align*}
&h^2(P_0^n \times Q_0^n, \overline{P^n} \times Q_0^n) \le \gamma < 2\\
&T(p_0,q_0) \ge 2\beta + T(p_\lambda, q) \forall \lambda \in \Lambda
\end{align*}
}
Then,
\begin{align}
\inf_{\hat{T}_n} \sup_{p,q \in \Theta} \PP_{p,q}^n\left[|\hat{T}_n - T(p,q)| > \beta\right] \ge c_\gamma
\end{align}
where $c_\gamma = \frac{1}{2}[1 - \sqrt{\gamma (1-\gamma/4)}]$.
\end{lemma}

Equipped with the above lemma, we can lower bound the rate of convergence by constructing densities $p_\lambda$ satisfying the bounded variation assumption, checking that they are well separated in the $L_2^2$ divergence sense, and bounding the hellinger distance.
We use the same construction as Krishnamurthy et al. and can therefore apply their hellinger distance bound (which is originally from Birge and Massart~\cite{birge1995estimation}).



We defer verifying the bounded variation assumption and the separation in $L_2^2$ divergence to the appendix as the arguments are a fairly technical and require several new definitions.
There, we show that the functions $p_\lambda$ can be chosen to belong to $\Wcal_1^\beta(C)$, have separation $\beta = n^{-\frac{4\beta}{4\beta+d}}$ (in absolute error), with $\gamma = O(1)$, resulting in the desired lower bound. 
The $n^{-1}$ term in the lower bound follows from a standard application of Le Cam's method (See Krishnamurthy et al.~\cite{krishnamurthy2014nonparametric}).

\ifthenelse{\equal{\version}{arxiv}}{\section{Experiments}}{\section{EXPERIMENTS}}
\ifthenelse{\equal{\version}{arxiv}}{
\begin{figure*}[t]
\begin{center}
\includegraphics[scale=0.25]{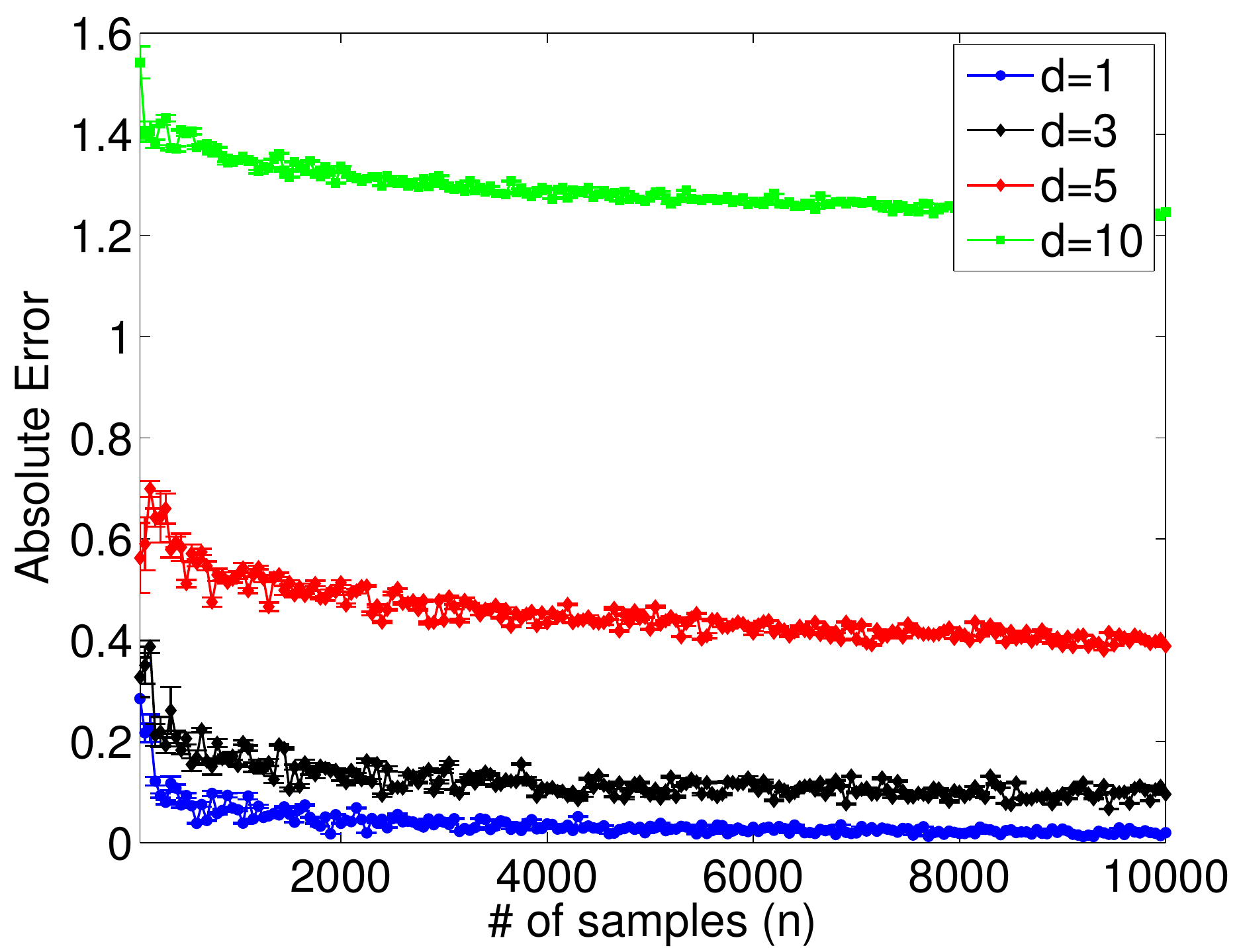}
\includegraphics[scale=0.25]{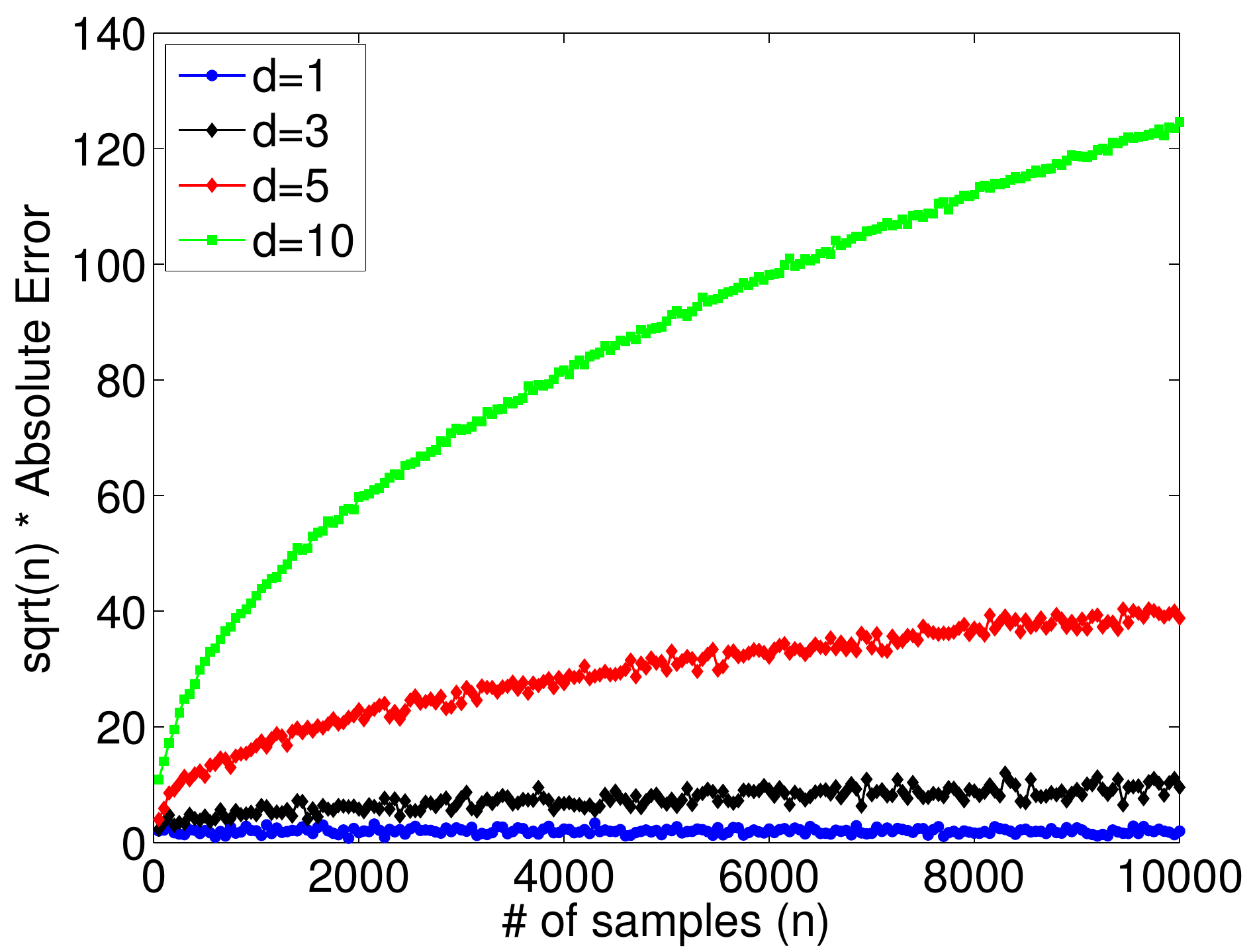}
\includegraphics[scale=0.25]{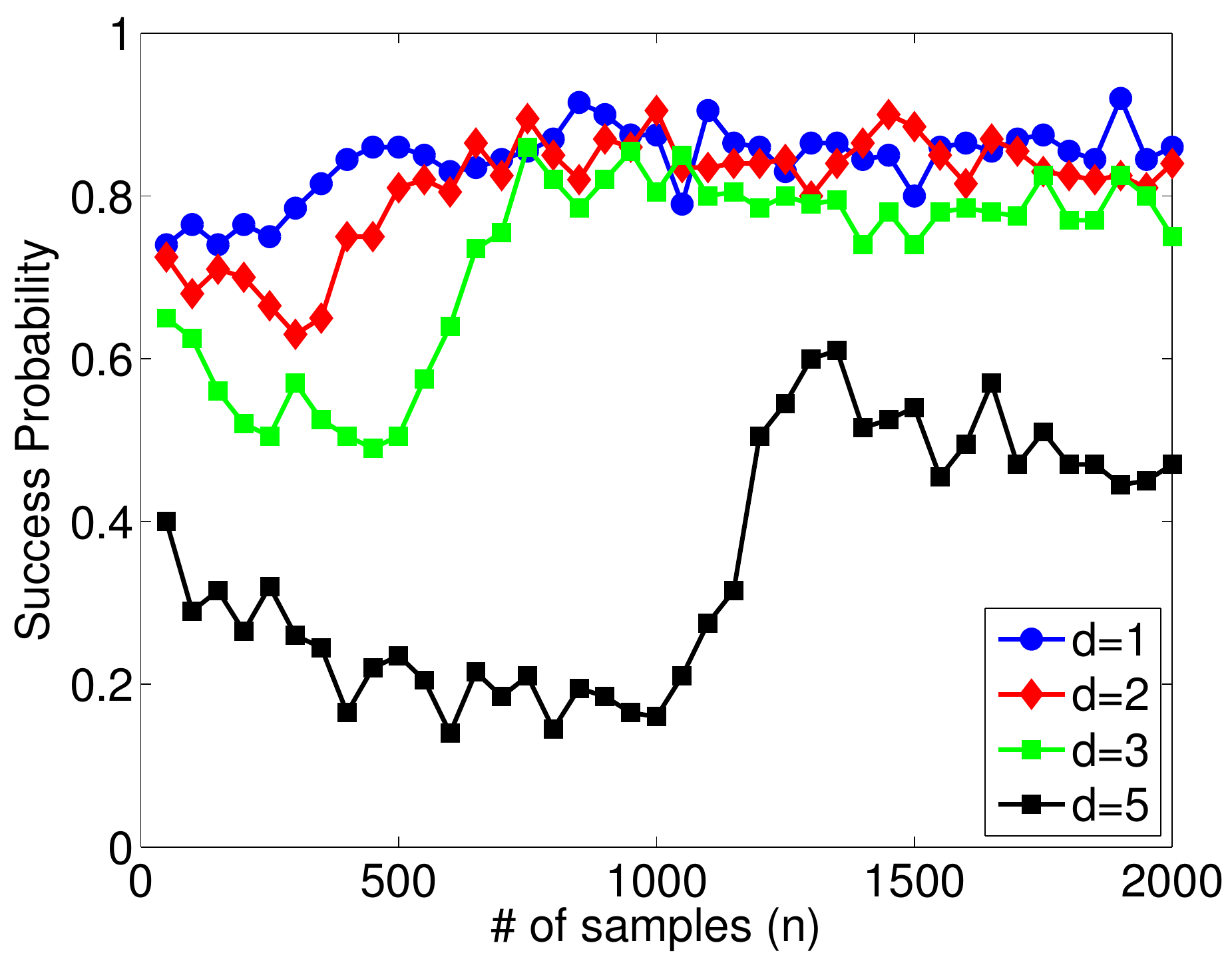}
\end{center}
\caption{Simulation results showing the convergence rate of the error, rescaled convergence rate, and performance of the confidence interval (from left to right).}
\label{fig:simulations}
\end{figure*}
}{
\begin{figure*}[t]
\includegraphics[scale=0.28]{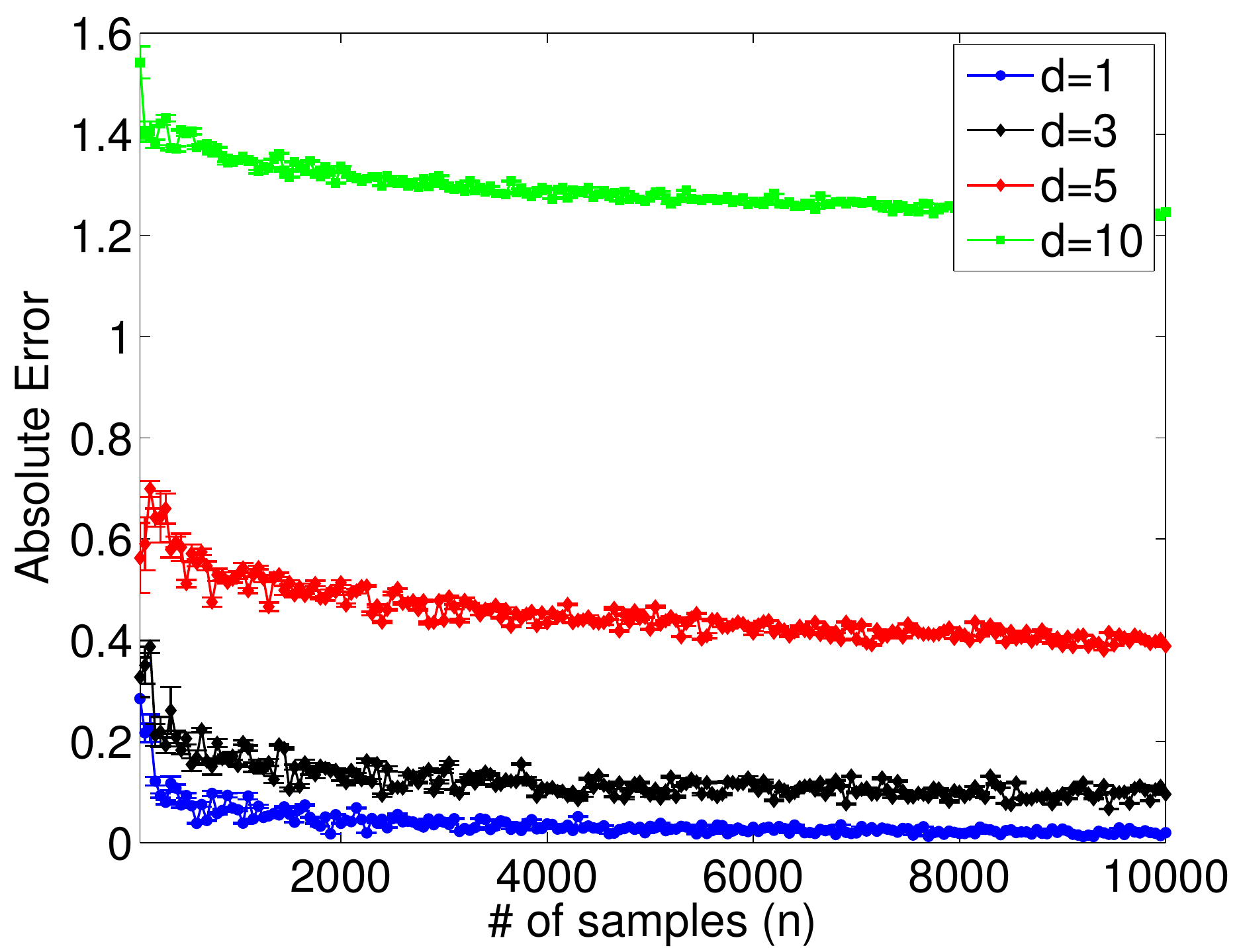}
\includegraphics[scale=0.28]{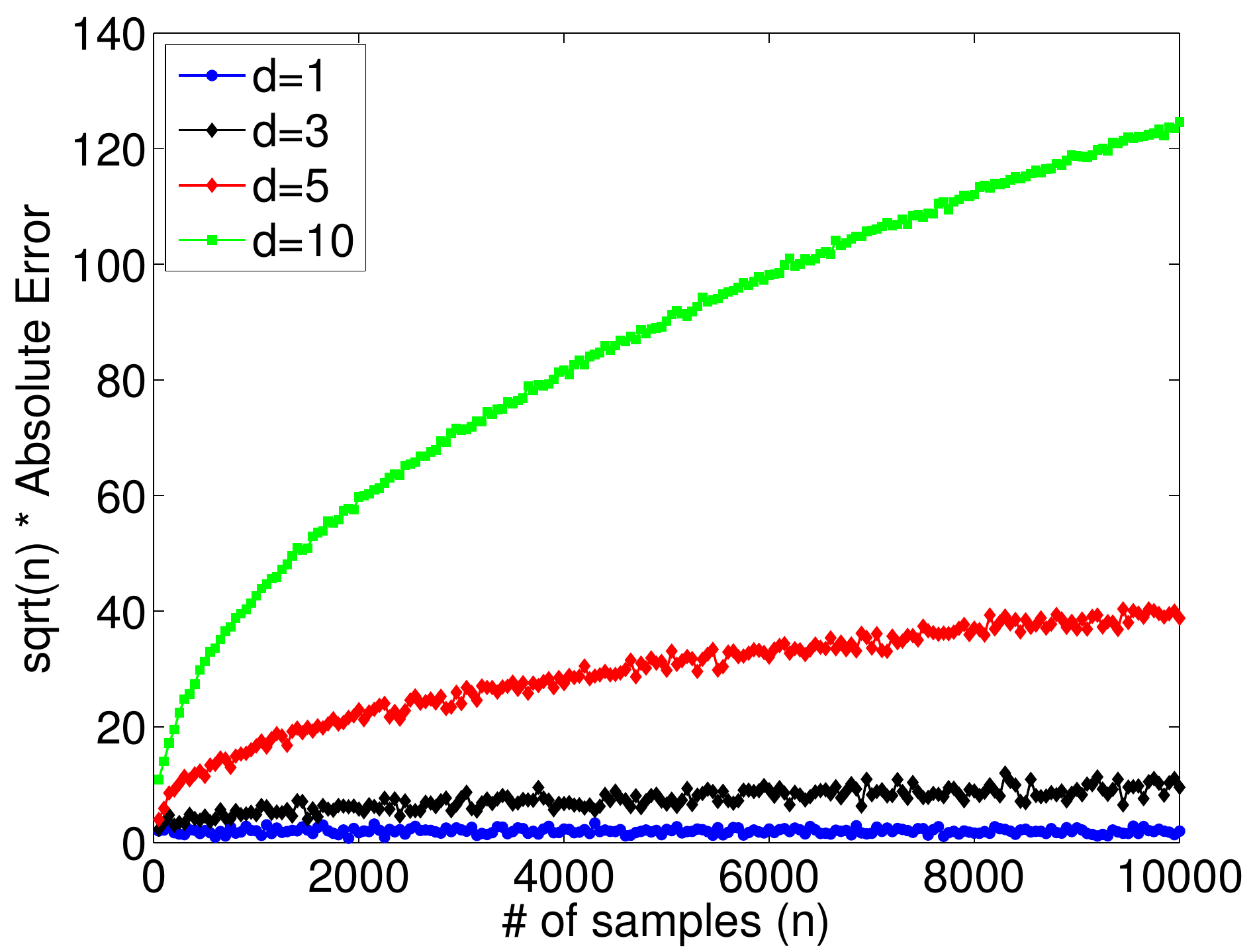}
\includegraphics[scale=0.28]{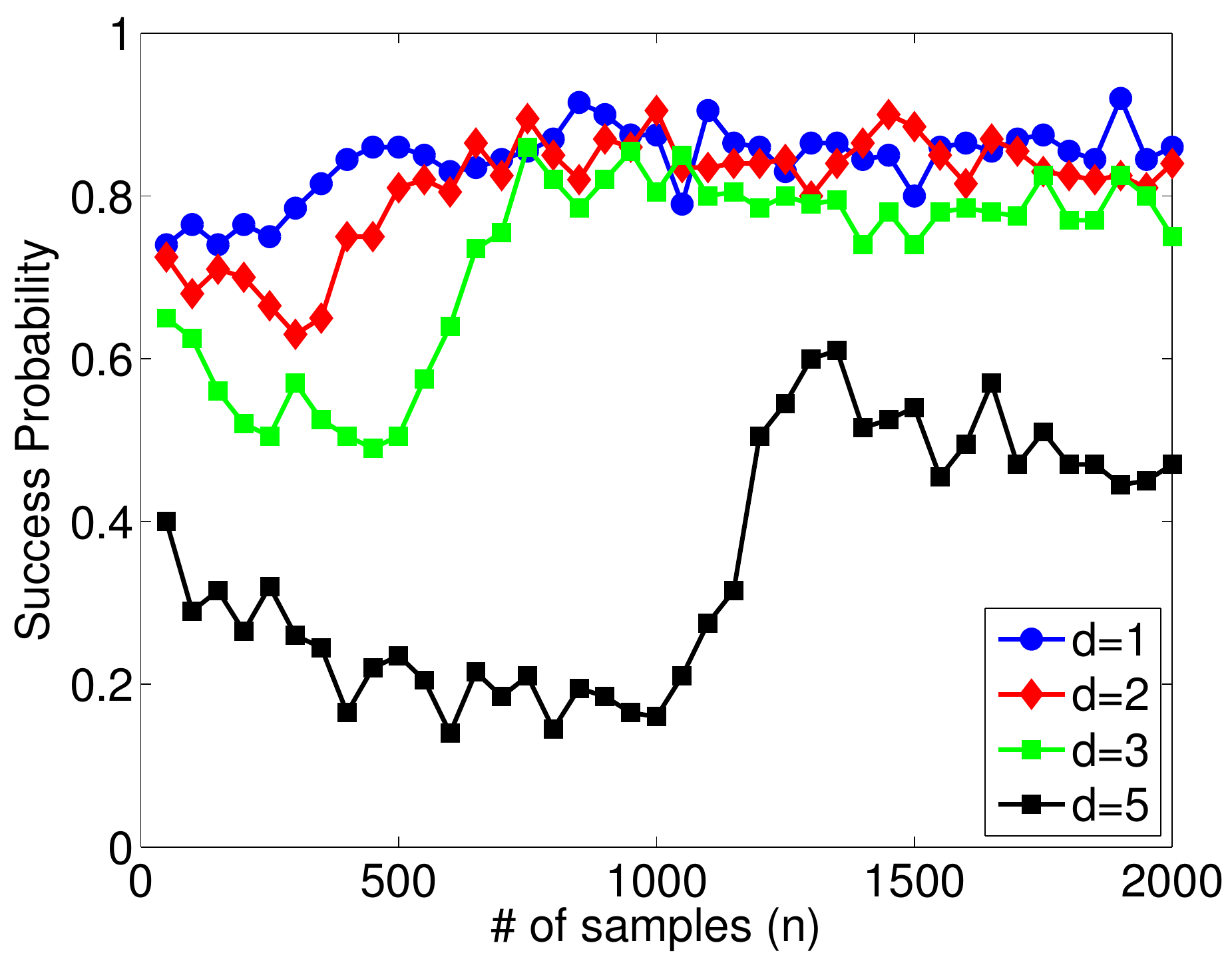}\\
\caption{Simulation results showing the convergence rate of the error, rescaled convergence rate, and performance of the confidence interval (from left to right).}
\label{fig:simulations}
\end{figure*}
}
The results of our simulations are in Figure~\ref{fig:simulations}.
For the first two plots, we trained our estimator on data generated from two Gaussian with means $(0,\ldots,0) \in \RR^d$ and $(1,\ldots,1) \in \RR^d$.
Note that the true $L_2^2$ distance can be analytically computed and is $\frac{2}{(2 \sqrt{\pi})^d} \left(1 - e^{-d/4}\right)$.
The bandwidth is chosen to scale appropriately with the number of samples and we use a Gaussian kernel.
Observe also that the Gaussian distribution satisfies the bounded variation assumption for any $\beta$ but that the Gaussian kernel does not meet our kernel requirements. 

In the first plot, we record the relative error $\frac{|\hat{D} - D|}{D^{-1}}$ of the estimator as a function of the number of samples for four different problem dimensions.
We use relative error in this plot to ensure that the curves are on the same scale, as the $L_2$-divergence between Gaussians decreases exponentially with dimension.
In the second plot, we rescale the relative error by $\sqrt{n}$. 

The first plot shows that the error is indeed converging to zero and that the relative error increases with dimension.
In the second plot, we see that the rescaled error curves all flatten out, confirming the $n^{-1/2}$ convergence rate in the $\ell_1$ metric.
However, notice that both the asymptote and the sample size at which the curves flatten out is increasing with dimension.
The latter suggests that, in high dimension, one needs a large number of samples before the $\sqrt{n}$-rate comes into effect.
It also suggests that there may be curse-of-dimensionality effect that is not captured by our analysis, as we think of $d$ as fixed throughout. 


In the third plot, we explore the empirical properties of our confidence interval.
As before, we generate data from two Gaussian distributions, compute the confidence interval and record whether the interval traps the true parameter or not.
In the figure, we plot the empirical probability that the $90\%$ confidence interval traps the true parameter as a function of the number of samples.
In low dimension, the confidence interval seems to be quite accurate as the empirical probability approaches $90\%$. 
However, even in moderate dimension, the confidence interval is less effective, as the sample size is too small for the asymptotic approximation to be accurate.
This is confirmed by the previous figure, as the sample size must be quite large for the $\sqrt{n}$-asymptotics to take effect.

While we are not aware of better confidence intervals in the general setting, significant improvement is possible in the special case of two-sample testing, where only a confidence interval around the null hypothesis of $p=q$ is necessary.
Here, rather than using the $1-\alpha/2$th quantile of the asymptotic distribution for designing the test, we recommend performing a permutation test, which gives an exact confidence interval under the null.
Of course this is also possible with the MMD, and indeed this is the recommended MMD-based two-sample test procedure~\cite{gretton2012kernel}.

Since one does not appeal to the limiting distribution in a permutation test, it has an added benefit of not requiring data splitting between the squared and cross terms of the $L_2^2$-divergence estimator.
While data-splitting played essentially no role our analysis, it leads to a noticeable decrease in power empirically.
Unfortunately, in the more general setting where one wants a confidence interval for $D(p,q)$, we are not aware of a better approach than our proposal.

\ifthenelse{\equal{\version}{arxiv}}{\section{Discussion}}{\section{DISCUSSION}}
\label{sec:discussion}

In this paper we studied a simple estimator for the $L_2^2$ divergence between continuous distributions. 
We showed that the estimator achieves the parametric $\sqrt{n}$ rate of convergence as soon as the densities have $d/4$-orders of smoothness. 
We also proved asymptotic normality, derived an asymptotic confidence interval, and characterized the quality of the asymptotic approximation with a Berry-Ess\'{e}en style inequality. 
Lastly we used information theoretic techniques to show that our estimator achieves the minimax optimal rate.
This gives a thorough characterization of the theoretical properties of this estimator. 



While our theoretical results are quite comprehensive, a number of questions still remain. 
First, despite enjoying a $\sqrt{n}$-rate in a fixed-dimension analysis, our simulations suggest that the performance degrades drastically with dimension.
This phenomenon is worth investigating further, both for our estimator and for other nonparametric methods.
It may be the case that fixed dimension asymptotic arguments are not appropriate for these semiparametric problems, as they hide a curse of dimensionality phenomenon.


It is also worth exploring how the $L_2^2$ divergence estimator and other nonparametric functionals can be used algorithmically in learning problems. 
One challenging problem involves optimizing a nonparametric functional over a finite family of distributions in an active learning setting (for example, finding the closest distribution to a target).
Here the so-called Hoeffding racing algorithm, which carefully constructs confidence intervals and focuses samples on promising distributions, has been used in the discrete setting with considerable success~\cite{loh2013faster}.
This algorithm relies heavily on exact finite sample confidence intervals that are largely absent from the nonparametrics literature, so extension to continuous distributions would require new theoretical developments.

Regarding two sample testing, an important open question is to identify which test statistic is best for a particular problem.
To our knowledge, little progress has been made in this direction.

\ifthenelse{\equal{\version}{arxiv}}{}{We hope to explore these directions in future work. }

\section*{Acknowledgements}
This research is supported by DOE grant DESC0011114, NSF Grants DMS-0806009, IIS1247658, and IIS1250350, and Air Force Grant FA95500910373.
AK is supported in part by a NSF Graduate Research Fellowship.
AK would also like to thank Arthur Gretton and Aaditya Ramdas for several fruitful discussions. 

\bibliography{bibliography}
\bibliographystyle{plain}

\appendix
\section{Proof of Theorems~\ref{thm:rate} and~\ref{thm:normality}}
We analyze the two estimators separately and the proof of Theorem~\ref{thm:rate} follows immediately from Theorems~\ref{thm:quad_rate} and Theorems~\ref{thm:bilinear_rate} below. 

For the quadratic term estimators, we make a slight modification to a theorem from Gine and Nickl~\cite{gine2008simple}. 
The only difference between our proof and theirs is in controlling the bias, where we use the bounded-variation assumption while they use a Sobolev assumption. 
However this has little bearing, as the bias is still of the same order, and we have the following theorem characterizing the behavior of the quadratic estimator:

\begin{theorem}[Adapted from~\cite{gine2008simple}]
\label{thm:quad_rate}
Under Assumption~\ref{assump:main}, we have:
\begin{align}
\left| \EE[\hat{\theta}_p] - \theta_p \right| \le c_b h^{2\beta} \qquad
\EE\left[ (\hat{\theta}_p - \EE[\theta_p])^2\right] \le c_v \left(\frac{1}{n} + \frac{1}{n^2h^d}\right),
\label{eq:quadratic_bias_variance}
\end{align}
and  when $\beta > d/4$:
\begin{align}
\sqrt{n}(\hat{\theta}_p - \theta_p) \leadsto \Ncal(0, 4\Var_{x \sim p}(p(x))).
\end{align}
\end{theorem}

While we are not aware of any analyses of the bilinear term, it is not particularly different from the quadratic term, and we have the following theorem:
\begin{theorem}
\label{thm:bilinear_rate}
Under Assumption~\ref{assump:main}, we have:
\begin{align}
\left| \EE[\hat{\theta}_{pq}] - \theta_{pq} \right| \le c_b h^{2\beta} \qquad
\EE\left[ (\hat{\theta}_{pq} - \EE[\theta_{pq}])^2\right] \le c_v \left(\frac{1}{n} + \frac{1}{n^2h^d}\right),
\label{eq:bilinear_bias_variance}
\end{align}
and when $\beta > d/4$:
\begin{align}
\sqrt{n}(\hat{\theta}_{pq} - \theta_{pq}) \leadsto \Ncal(0, \Var_{x \sim p}(q(x)) + \Var_{y \sim q}(p(y))).
\end{align}
\end{theorem}

\begin{proof}[Proof of Theorem~\ref{thm:quad_rate}]
We reproduce the proof of Gine and Nickl for completeness.
The bias can be bounded by:
\begin{align*}
\EE[\hat{\theta}_p] - \theta_p &= \int \int K_h(x,y)p(y)dy p(x)dx - \int p(x)p(x)dx
= \int \int K_h(x,y)[p(y) - p(x)]p(x) dy dx\\
& = \int \int K(u)[p(x-uh) - p(x)] p(x) du dx 
= \int K(u) \left[ (p_0 \star p)(uh) - (p_0\star p)(0)\right]du,
\end{align*}
where $p_0(x) = p(-x)$ and $\star$ denotes convolution. 
Now by Lemma~\ref{lem:convolution_smoothness} below, we know that $p_0 \star p \in \Wcal_1^{2\beta}(C^2)$ and can take a Taylor expansion of order $2\beta -1$. 
When we take such an expansion, by the properties of the kernel, all but the remainder term is annihilated and we are left with:
\begin{align*}
\frac{h^{2\beta}}{(2\beta)!} \sum_{r_1, \ldots, r_d | \sum_i r_i = 2s} \int K(u) \Pi_{i} u_i^{r_i} \xi(r,uh)du \le c_b h^{2\beta},
\end{align*}
where we used the fact the function is integrable by the fact that $\xi \in L^1$, which in turn follows from the fact that $p_0 \star p \in \Wcal_1^{2\beta}(C^2)$ and by Taylor's remainder theorem.
We are also using the compactness of $K$ here so that we only have to integrate over $(-1,1)^d$ in which case all polynomial functions are also $L_1$ integrable. 
This shows that the bias is $O(h^{2\beta})$. 

Note that the main difference between our proof and that of Gine and Nickl is in the smoothness assumption, which comes into play here. 
Under the bounded variation assumption, we were able to argue that smoothness is additive under convolution.
The same is true under the Sobolev assumption, and this property is exploited by Gine and Nickl in exactly the same way as we do here.
Unfortunately, H\"{o}lder smoothness is not additive under convolution, so the more standard assumption does not provide the semiparametric rate of convergence. 

As for the variance, we may write:
\begin{align*}
\EE[\hat{\theta}_p^2] - (\EE \hat{\theta}_p)^2 = \EE\left[\frac{1}{n^2(n-1)^2}\sum_{i \ne j, s \ne t} K_h(X_i, X_j) K_h(X_s, X_t)\right] - (\EE\hat{\theta}_p)^2,
\end{align*}
which we can split into three cases. 
When $i \ne j \ne s \ne t$, each term in the sum is exactly $(\EE\hat{\theta})^2$, and this happens for $n(n-1)(n-2)(n-3)$ terms in the sum. 
When one of the first indices is equal to one of the second indices we get:
\begin{align*}
\EE K_h(X_i, X_j) K_h(X_i, X_t) &= \int \int \int K_h(X_i, X_j) K_h(X_i, X_t)p(X_i) p(X_j) p(X_t) dX_i dX_j dX_t\\
& = \int \int \int K(u_j)K(u_t) p(X_i - u_j h) p(X_i - u_th) p(X_i) du_j du_t dX_i\\
& \le ||K||_2^2 ||p||_2^2,
\end{align*}
where we performed a substitution to annihilate the dependence on $h$. 
There are $4n(n-1)(n-2)$ expressions of this form, so in total, these terms contribute:
\[
\frac{1}{n}||K||_2^2 ||p||_2^2.
\]
Finally, the $2n(n-1)$ terms where $i=s, j=t$ or vice versa in total contribute: 
\begin{align*}
\frac{2}{n(n-1)} \EE K_h^2(X_i, X_j) &= \frac{2}{h^{2d} n(n-1)} \int K^2(\frac{X_i - X_j}{h}) p(X_i) p(X_j)dX_i dX_j\\
& = \frac{2}{h^d n(n-1)} \int K^2(u_j)p(X_i) p(X_i - u_j h) du_j dX_i\\
& \le \frac{2 ||K||_2^2 ||p||_2^2}{h^d n^2}.
\end{align*}
Adding together these terms, establishes the variance bound in the theorem. 
The rate of convergence in Theorem~\ref{thm:rate} follows from plugging the definition of $h$, which was selected to optimize the tradeoff between bias and variance. 

As for asymptotic normality, we decompose the proof into several steps.
\begin{packed_enum}
\item Control the bias.
\item Apply Hoeffding's decomposition.
\item Control the second order term, which will be lower order.
\item Show that the first order term is close to $P_n p - \theta$ (here $P_n$ is the empirical measure).
\item Apply the Lindberg-Levy central limit theorem to $P_n p - \theta$. 
\end{packed_enum}
As usual we have the decomposition:
\begin{align*}
\hat{\theta}_p - \theta_p = \underbrace{\hat{\theta}_p - \EE\hat{\theta}_p}_{\textrm{Variance}} + \underbrace{\EE\hat{\theta}_p - \theta}_{\textrm{Bias}}.
\end{align*}

We already controlled the bias above.
Specifically we know that $\sqrt{n}(\EE\hat{\theta}_p- \theta_p) \le \sqrt{n}h^{2\beta} \rightarrow 0$ with our setting of $h$ and under the assumption that $\beta > d/4$. 

As is common in the analysis of U-statistics, we apply Hoeffding's decomposition before proceeding.
That is, we write:
\begin{align*}
\hat{\theta}_p - \EE\hat{\theta}_p = U_n (\pi_2 K_h) + 2 P_n(\pi_1K_h),
\end{align*}
where $U_n f = \frac{1}{n(n-1)}\sum_{i \ne j}f(X_i, X_j)$ is the U-process and $P_n f = \frac{1}{n}\sum_{i} f(X_i)$ is the empirical process and:
\begin{align*}
(\pi_1 K_h)(X) &= \EE_{x \sim p}K_h(x,X) - \EE_{x,y\sim p}K_h(x,y)\\
(\pi_2 K_h)(X,Y) &= K_h(X,Y) - \EE_{x \sim p}K_h(x,Y) - \EE_{y \sim p} K_h(X,y) + \EE_{x,y \sim p}K_h(x,y).
\end{align*}
It is easy to very that our estimator can be decomposed in this manner. 
Moreover, since everything is centered, the two terms also have zero covariance.
Also notice that $\EE_{x \sim p}K_h(x,Y) = \bar{p}(Y)$ and $\EE_{x,y \sim p}K_h = \int \bar{p}(x)p(x)$ where $\bar{p}$ is the expectation of the density estimate.

We now control the second order term $U_n (\pi_2 K_h)$ by showing convergence in quadratic mean. 
\begin{align*}
\EE[(U_n (\pi_2K_h))^2] = \frac{1}{n(n-1)} \EE[(\pi_2K_h(X_1, X_2))^2] \le \frac{c}{n^2h^d}||K||_2^2 ||p||_2^2.
\end{align*}
The first equality follows from the fact that each term is conditionally centered, so all cross terms are zero, while the inequality is the result of performing a substitution as we have seen before. 
Thus $\sqrt{n}U_n(\pi_2K_h) \rightarrow 0$ since $\frac{1}{nh^d} \rightarrow 0$ when $\beta > d/4$. 

For the first order term $P_n (\pi_1K_h)$, we now show that it is close to $P_n p - \theta_p$. 
\begin{align*}
\EE[(P_n(\pi_1 K_h) - (P_n p - \int p^2))^2] \le \frac{1}{n} \EE [(\bar{p}(X) - p(x))^2] \le \frac{||\bar{p} - p||_{\infty}^2}{n} = \frac{c h^{2\beta}}{n},
\end{align*}
so that $\sqrt{n}P_n (\pi_1K_h) \rightarrow^{q.m.} \sqrt{n} (P_n p - \theta_p)$ since $h^{2\beta} \rightarrow 0$. 

Now by the Lindberg-Levy CLT, we know that:
\[
\sqrt{n} (2P_n p - 2\theta_p) \leadsto \Ncal(0, 4 \Var_{x \sim p} (p(X))),
\]
which concludes the proof of the theorem.
\end{proof}

We now prove Theorem~\ref{thm:bilinear_rate}, although the arguments are fairly similar. 
\begin{proof}[Proof of Theorem~\ref{thm:bilinear_rate}]
The bias is:
\begin{align*}
& \EE[\hat{\theta}_{pq}] - \theta_{pq} = \int \int K_h(x,y)p(y)dy q(x)dx - \int p(x)q(x)dx\\
& = \int \int K_h(x,y)[p(y) - p(x)]q(x) dy dx\\
& = \int \int K(u)[p(x-uh) - p(x)] q(x) du dx = \int K(u) \left[ (p_0 \star q)(uh) - (p_0\star q)(0)\right]du,
\end{align*}
where, as before, $p_0(x) = p(-x)$ and $\star$ denotes convolution. 
So we can proceed as in the quadratic setting. 
Specifically, by Lemma~\ref{lem:convolution_smoothness}, we can take a Taylor expansion of order $2\beta+1$, annihilate all but the remainder term, which we know is bounded by the fact that $p_0\star q \in \Wcal_1^{2\beta}(C^2)$.
Formally, the remainder term is:
\begin{align*}
\frac{h^{2\beta}}{(2\beta)!} \sum_{r_1, \ldots, r_d | \sum_i r_i = 2s} \int K(u) \Pi_{i} u_i^{r_i} \xi(r,uh)du \le c_b h^{2\beta},
\end{align*}
where we used the fact the function is integrable by the fact that $\xi \in L^1$, since $p_0 \star q \in \Wcal_1^{2\beta}(C^2)$.
Thus the bias is $O(h^{2\beta})$. 

The variance can be bounded in a similar way to the quadratic estimator:
\begin{align*}
\EE[\hat{\theta}_{pq}^2] - \EE[\hat{\theta}_{pq}]^2 = \frac{1}{n^4}\sum_{i,j,s,t} \EE[K_h(X_i, Y_j)K_h(X_s,Y_t)] - \EE[\hat{\theta}_{pq}]^2.
\end{align*}
Whenever $i \ne s$ and $j \ne t$ all of the terms are independent so they cancel out with the $\EE[\hat{\theta}_{pq}]^2$ term. 
This happens for $n^2(n-1)^2$ terms. 

When $i=s, j \ne t$, we substitute $u_j = h^{-1} (X_i - Y_j)$ and $u_t = h^{-1}(X_i - Y_t)$ for $Y_j,Y_t$ to see that:
\begin{align*}
\frac{1}{n^4} \sum_{i, j \ne t} \EE[K_h(X_i, Y_j)K_h(X_i, Y_t)] &= \frac{n-1}{h^{2d} n^2} \int \int \int K(\frac{X_i - Y_j}{h}) K(\frac{X_i - Y_t}{h})p(X_i)q(Y_j)q(Y_t)\\
& = \frac{n-1}{n^2}\int \int \int K(u_j)K(u_t) p(X_i)q(X_i - u_jh) q(X_i - u_th)\\
& \le \frac{1}{n} ||K||_2^2 ||q||_2^2.
\end{align*}
Thus, the total contribution from the terms where $j = t, i \ne s$ is bounded by $\frac{1}{n} ||K||_2^2 ||p||_2^2$. 

When $j=t,i=s$, we can only perform one substitution so a factor of $h^d$ will remain. 
Formally:
\begin{align*}
\frac{1}{n^2h^{2d}} \int \int K^2(\frac{X_i - Y_j}{h}) p(X_i) q(Y_j) &= \frac{1}{n^2h^d} \int \int K^2(u_j) p(X_i) q(X_i - u_jh)\\
& \le \frac{1}{n^2h^d} ||K||_2^2 ||p||_2 ||q||_2.
\end{align*}
Therefore, the total variance is $O(n^{-1} + n^{-2}h^{-d})$ as in the theorem statement. 

The proof of asymptotic normality of the bilinear estimator is not too different from the proof for the quadratic estimator. 
We can start by ignoring the bias, as when $b \ge d/4$, we know that $\sqrt{n}(\EE\hat{\theta}_{pq} - \theta_{pq}) \rightarrow 0$.
To analyze the variance term we make use of the following decomposition:
\begin{align*}
\hat{\theta}_{pq} - \EE \hat{\theta}_{pq} &= \frac{1}{n^2}\sum_{ij} K_h(X_i, Y_j) + \frac{1}{n} \sum_{i}\bar{q}(X_i) - \frac{1}{n}\sum_{i} \bar{q}(X_i) + \frac{1}{n}\sum_{j} \bar{p}(Y_i) - \frac{1}{n}\sum_{j} \bar{p}(Y_i) - \EE\hat{\theta}_{pq}\\
& = V_n(\pi_2 K_h) + P_n (\pi_{11}K_h) + Q_n(\pi_{12}K_h),
\end{align*}
Where:
\begin{align*}
(\pi_2K_h)(X,Y) &= K_h(X,Y) - \bar{q}(X) - \bar{p}(Y) + \EE K_h(x,y)\\
(\pi_{11}(K_h))(X) & = \bar{q}(X) - \EE K_h(x,y)\\
(\pi_{12}(K_h))(Y) & = \bar{p}(Y) - \EE K_h(x,y).
\end{align*}
Here $P_n, Q_n$ are the empirical processes associated with the samples $X, Y$ respectively and $\bar{p}, \bar{q}$ are the expectations of the kernel density estimators. 
Also, $V_n$ is the $V$-process, that is $V_n f = \frac{1}{n^2} \sum_{i,j} f(X_i, Y_j)$. 
Notice that each term is conditionally centered, which implies that each pair of terms has zero covariance.
Thus we only have to look at the variances. 

As before, the goal is to show that the $V$-process term is lower order and then to apply the Lindeberg-Levy CLT to the other two terms. 
Since each term is conditionally centered:
\begin{align*}
\EE[ (V_n (\pi_2 K_h))^2 ] = \frac{1}{n^2} \EE[(\pi_2K_h(X,Y))^2] \le \frac{c}{n^2h^d}||K||_2^2 ||p||_2 ||q||_2,
\end{align*}
where the last step follows by performing the substitution $u = \frac{X-Y}{h}$ in each term of the integral. 

For the first order terms, we first show that they are close to $q(X) - \EE[q(x)]$ and $p(Y) - \EE[p(y)]$ so that we can apply the CLT to the latter.
We will show convergence in quadratic mean.
\begin{align*}
\EE\left[ (P_n(\pi_{11}K_h) - (P_n q - \int pq))^2\right] \le \frac{1}{n} \EE\left[(\bar{q}(x) - q(x))^2\right] \le \frac{\|\bar{q} - q\|_{\infty}^2}{n} \le \frac{h^{2\beta}}{n},
\end{align*}
which means that, under our choice of $h$ and with $\beta > d/4$, $\sqrt{n}P_n (\pi_{11}K_h) \rightarrow^{q.m.} \sqrt{n}P_n(q - \theta_{pq})$.
Exactly the same argument shows that $\sqrt{n}Q_n (\pi_{12}K_h) \rightarrow^{q.m.} \sqrt{n}Q_n(p - \theta_{pq})$.

Finally, by the Lindeberg-Levy CLT, we know that:
\begin{align*}
\sqrt{n}(P_n q - \theta_{pq}) \leadsto \Ncal(0, \Var_{x \sim p}(q(x))), \qquad \sqrt{n}(Q_n p - \theta_{pq}) \leadsto \Ncal(0, \Var_{y \sim q}(p(y))),
\end{align*}
and since $x$ and $y$ are independent, both of these central limit theorems hold jointly. 
Since in our estimate for $\hat{D}$ we have a term of the form $2 \hat{\theta}_{pq}$, the contribution of this term to the total variance is $4 \Var(\hat{\theta}_{pq})$.
This concludes the proof. 
\end{proof}

\section{Proof of Theorem~\ref{thm:berry}}
In this section we fill in the missing details in the proof of Theorem~\ref{thm:berry}. 
We will apply the Berry-Ess\'{e}en inequality for multi-sample U-statistics from Chen, Goldstein and Shao~\cite{chen2010normal}, which we reproduce below.

In order to state the theorem we need to make several definitions.
We make some simplifications to their result for ease of notation.
Consider $k$ independent sequences $X_{j1}, \ldots, X_{jn}$ $j=1, \ldots, k$ of i.i.d. random variables, all of length $n$ (this can be relaxed). 
Let $m_j \ge 1$ for each $j$ and let $\omega(x_{jl}, l \in [m_j], j \in [k])$ be a function that is symmetric with respect to the $m_j$ arguments.
In other words, $\omega$ is invariant under permutation of two arguments from the same sequence. 
Let $\theta = \EE \omega(X_{jl})$. 

The multi-sample U-statistic is defined as:
\begin{align}
U_n = \left\{ \prod_{j=1}^k {n \choose m_j}^{-1}\right\}\sum \omega(X_{jl}, j\in [k], l=i_{j1}, \ldots, i_{jm_j}),
\end{align}
where the sum is carried out over all indices satisfying $1 \le i_{j1} < \ldots < i_{jm} \le n$. 

Let:
\[
\sigma^2 = \EE \omega^2(X_{jl}),
\]
and for each $j \in [k]$ define:
\[
\omega_j(x) = \EE [ \omega(X_{jl}) | X_{j1} = x],
\]
with:
\[
\sigma^2_j = \EE \omega_j^2(X_{j1}).
\]
Lastly, define:
\[
\sigma^2_n = \sum_{j=1}^k \frac{m_j^2}{n} \sigma_j^2.
\]
We are finally ready to state the theorem:
\begin{theorem}[Theorem 10.4 of~\cite{chen2010normal}]
\label{thm:u_berry}
Assume that $\theta = 0$, $\sigma^2 < \infty$, $\max_{j \in [k]} \sigma_j^2 > 0$. Then for $2 < p \le 3$:
\begin{align}
\sup_{z \in \RR} \left| \PP\left(\sigma_n^{-1} U_n \le z\right) - \Phi(z) \right| \le \frac{6.1}{\sigma_n^p}\sum_{j=1}^k \frac{m_j^p}{n^{p-1}}\EE [|\omega_j(X_{j1})|^p] + \frac{(1+\sqrt{2})\sigma}{\sigma_n} \sum_{j=1}^k \frac{m_j^2}{n}.
\end{align}
\end{theorem}

As we did in our estimator, we split the data into four groups, two samples of size $n$ from each distribution, which we will denote with superscripts, i.e. $X^{(1)}_i$ will be the $i$th sample from the first group of the data from $p$. 
We can write $\hat{D} - \EE \hat{D}$ as a zero-mean multi-sample U-statistic with four groups where the first $X$ and $Y$ groups are used for $\hat{\theta}_p - \theta_p$ and $\hat{\theta}_q - \theta_q$ respectively, while the second two groups are used for the cross term $\hat{\theta}_{pq}- \theta_{pq}$. 

In other words, $\omega$ will be a function that takes $6$ variables, two from the $X^{(1)}$ group, two from the $Y^{(1)}$ group and one each from the $X^{(2)}$ and $Y^{(2)}$ groups. 
Formally, we define:
\begin{align*}
&\omega(x_{11}, x_{12}, y_{11}, y_{12}, x_{21}, y_{21}) \\
&= K_h(x_{11},x_{12}) - \EE \hat{\theta}_{p} + K_h(y_{11},y_{12}) - \EE\hat{\theta}_q - 2 K_h(x_{21},y_{21}) + 2 \EE\hat{\theta}_{pq}.
\end{align*}
With this definition, it is clear that $U_n = \hat{D} - \EE\hat{D}$.

To apply Theorem~\ref{thm:u_berry} on the appropriate term in the proof, we just have to bound a number of quantities involving $\omega$. 
As we will see, we will not achieve the $n^{-1/2}$ rate because the function $\omega$ depends on the bandwidth $h$, which is decreasing, so the variance $\sigma$ is increasing.
Specifically:
\begin{align*}
\sigma^2 &= \EE[\omega^2(X_{11}, X_{12}, Y_{11}, Y_{12}, X_{21}, Y_{21})]\\
&= \EE (K_h(X_{11}, X_{12}) - \EE\hat{\theta}_p)^2 + \EE (K_h(Y_{11}, Y_{12}) - \EE\hat{\theta}_q)^2 + 4\EE (K_h(X_{21}, Y_{21}) - \EE\hat{\theta}_{pq})^2.
\end{align*}
Each of the three terms can be analyzed in exactly the same way so we focus on the first term:
\begin{align*}
\EE (K_h(X_{11}, X_{12}) - \EE\hat{\theta}_p)^2 &\le \int \int K^2_h(X_1, X_2)p(X_1)p(X_2) = \frac{1}{h^d}\int \int K(u)p(X_1 + uh)p(X_1) \\
& \le \frac{1}{h^d}\|K\|_2^2\|p\|_2^2,
\end{align*}
and the same substitution on the other two terms shows that the variance is:
\begin{align}
\sigma^2 \le \frac{1}{h^d}\|K\|_2^2\left(\|p\|_2^2 + \|q\|_2^2 + 4 \|p\|_2\|q\|_2\right).
\end{align}

A similar argument gives us a bound on $\sigma_j^2$ $j = 1, \ldots 4$. 
First, since the other terms are centered, we can write $\omega_1(x) = \EE (K_h(x, X_2)) - \EE \hat{\theta}_p$ with similar expressions for the other terms. 
Then, $\sigma_1^2$ can be simplified to:
\begin{align*}
\EE \omega^2_1(X) &= \EE (\EE K_h(X_1,X_2))^2 - (\EE\hat{\theta}_p)^2\\
& \le \int \left( \int K_h(X_1, X_2)p(X_2)\right)^2 p(X_1) = \int \left( \int K(u)p(X_1 - uh) \right)^2 p(X_1) \le \|K\|_{\infty}^2.
\end{align*}
With exactly the same argument for the other three.
Thus:
\begin{align}
\sigma_n^2 = \frac{1}{n}\sum_{j=1}^4 m_j^2 \sigma_j^2 \le \frac{10}{n}\|K\|_{\infty}^2.
\end{align}

The last thing we need is the third moments of the linearizations $\EE[|\omega_j(x)|^3]$. 
\begin{align*}
\EE\left[\left|\EE K_h(X_1, X_2) - \EE\hat{\theta}_{p}\right|^3\right] & = \int \left| \int K_h(x, X_2) p(X_2) - \int \int K_h(X_1, X_2)p(X_1)p(X_2)\right|^3 p(x)\\
& = \int \left| \int K(u) p(x- uh) - \int \int K(u) p(X_1 - uh)p(X_1)\right|^3 p(x)\\
& = \int \left| \int K(u) \left( p(x-uh) - \int p(X_1-uh)p(X_1)\right) \right|^3 p(x)\\
& \le 8 \|K\|_{\infty}^3 \|p\|_{\infty}^3
\end{align*}
It is easy to verify that each of the third moments are bounded by:
\begin{align}
\EE|\omega_j|^3 \le 8 \|K\|_{\infty}^3 (\|p\|_{\infty}^3+\|q\|_{\infty}^3),  \forall j.
\end{align}

And plugging in all of these calculations into Theorem~\ref{thm:u_berry} shows that:
\begin{align*}
& \sup_z\left| \PP\left(\sqrt{n}\tilde{\sigma}_n^{-1}U_n \le z\right) - \Phi(z)\right| \le\\
& \le \frac{n^{3/2}(6.1)(18)}{n^2 10^{3/2}\|K\|_{\infty}^3}8 \|K\|_{\infty}^3 (\|p\|_{\infty}^3+\|q\|_{\infty}^3) + \frac{\sqrt{n}(1+\sqrt{2})}{n\sqrt{h^d}\sqrt{10} \|K\|_{\infty}}\|K\|_2\sqrt{\|p\|_2^2 + \|q\|_2^2 + 4 \|p\|_2\|q\|_2}\\
& \le \frac{27}{\sqrt{n}} \left(\|p\|_{\infty}^3 + \|q\|_{\infty}^3\right) + \frac{8}{\sqrt{n h^d}}\frac{\|K\|_2}{\|K\|_{\infty}}\sqrt{\|p\|_2^2 + \|q\|_2^2 + 4 \|p\|_2\|q\|_2}.
\end{align*}
This gives the bound in Equation~\ref{eq:berry_app}.

\section{Proof of Theorem~\ref{thm:l2_lower}}

For completeness we introduce the construction used by Krishnamurthy et al~\cite{krishnamurthy2014nonparametric}.
For the remainder of the proof, we will work of $[0,1]^d$ and assume that $p$ is pointwise lower bounded by $1/\kappa_l$, noting that a lower bound here applies to the more general setting.
For the construction, suppose we have a disjoint collection of subset $A_1, \ldots, A_m \subset [0,1]^d$ for some parameter $m$ with associated functions $u_j$ that are compactly supported on $A_j$. 
Specifically assume that we have $u_j$ satisfying:
\begin{align*}
\textrm{supp}(u_j) \subset \{x | B(x,\epsilon \subset A_j\}, \|u_j\|_2^2 = \Omega(m^{-1}), \int_{A_j}u_j = \int_{A_j}p_0(x)u_j(x) = \int_{A_j}q_0(x)u_j(x) = 0, \|D^ru_j\|_1 \asymp m^{r/d-1}
\end{align*}
The first condition ensure that the $u_j$s are orthogonal to each other, while the second and third will ensure separation in terms of $L_2^2$ divergence. 
The last condition holds for all derivative operators with $r \le \beta$ and it will ensure that the densities we construct belong to the bounded variation class. 
The only difference between these requirements and those from~\cite{krishnamurthy2014nonparametric} are the orthogonality to $p,q$, and the bounded-variation condition, which replaces a point-wise analog.

Deferring the question of existence of these functions, we can proceed to construct $p_\lambda$. 
Let the index set $\Lambda = \{-1, +1\}^m$ and define the functions $p_\lambda = p_0 + K \sum_{j=1}^m \lambda_j u_j$, where $K$ will be defined subsequently. 
A simple computation then reveals that:
\begin{align*}
T(p_0,q_0) - T(p_\lambda, q_0) & = \int p_0^2 - p_\lambda^2 + 2\left[\int p_\lambda q_0 - \int p_0q_0\right] \\
& = \int (p_0 - p_\lambda) (p_0+p_\lambda) + 2\left[\int p_\lambda q_0 - \int p_0q_0\right] \\
& = K^2 \sum_{j=1}^m \|u_j\|_2^2 = \Theta(K^2)
\end{align*}
where we expand $p_\lambda$ and use the orthogonality properties extensively.
This gives us the desired separation.

To bound the hellinger distance, we use Theorem 1 of Birge and Massart~\cite{birge1995estimation} and the argument following Theorem 12 of Krishnamurthy et al~\cite{krishnamurthy2014nonparametric}.
\begin{theorem}{\cite{birge1995estimation}}
Consider a set of densities $p_0$ and $p_\lambda = p[1+\sum_j \lambda_j v_j(x)]$ for $\lambda \in \Lambda = \{-1,1\}^m$ with partition $A_1, \ldots, A_m \subset [0,1]^d$.
Suppose that (i) $\|v_j\|_{\infty} \le 1$, (ii) $\|\mathbf{1}_{A_j^C} v_j\|_1 = 0$, (iii) $\int v_j p_0 = 0$ and (iv) $\int v_j^2p_0 = \alpha_j > 0$ all hold with:
\begin{align*}
\alpha = \sup_j \|v_j\|_{\infty}, s= n\alpha^2 \sup_j P_0(A_j), c = n \sup_j \alpha_j
\end{align*}
Define $\overline{P^n} = \frac{1}{|\Lambda|}\sum_{\lambda \in \Lambda} P_\lambda^n$.
Then:
\begin{align}
h^2(P_0^n, \overline{P^n}) \le C(\alpha, s, c)n^2\sum_{j=1}^m\alpha_j^2
\end{align}
where $C < 1/3$ is continuous and non-decreasing with respect to each argument and $C(0,0,0) = 1/16$. 
\end{theorem}

The exact same bound on the hellinger distances holds for the measures $P_0^n\times Q_0^n$ against $\overline{P^n}\times Q^n$.
Defining $v_j = Ku_j/p_0$ then the densities we used in our construction meet the specification in the above theorem. 
We immediately satisfy the first three requirements and we have $\int v_j^2 p = K^2 \int u_j^2/p \le K^2 \kappa_l/m \triangleq \alpha_j$. 
Thus we have the hellinger bound of:
\begin{align*}
h^2(P_0^n\times Q_0^n,\overline{P^n}\times Q^n) \le (1/3)n^2\sum_{j=1}^m \alpha_j^2 \le \frac{Cn^2K^4}{m}
\end{align*}

We lastly have to make sure that the $p_\lambda$ functions satisfy the bounded variation assumption. 
This follows from an application of the triangle inequality provided that $\|D^ru_j\|_1 \le O(m^{r/d-1})$.
\begin{align*}
\|D^r p_\lambda\|_1 &= \|D^r p + K \sum_{j=1}^m \lambda_j D^ru_j\|_1 \le \|D^rp\| + K\sum_{j=1}^m \|D^ru_j\|_1 \le \|D^rp\| + K \sum_{j=1}^m \|D^ru_j\|_1
\le \|D^rp\| + O(Km^{r/d})
\end{align*}
So as long as $K \asymp m^{-r/d}$ and there is some wiggle room around the bounded variation assumption for $p$, $p_\lambda$ will meet the bounded variation assumption.

Before we construct the $u_j$s, we put everything together. 
We must select $K \asymp m^{-\beta/d}$ so that $p_\lambda \in \Wcal_1^\beta(C)$, and then to make the hellinger distance $O(1)$, we must set $m \asymp n^{\frac{2d}{4\beta+d}}$.
This makes $K^2 \asymp n^{\frac{-4\beta}{4\beta+d}}$ which is precisely the lower bound on the convergence rate in absolute error. 

Lastly we present the construction of the $u_j$ functions. 
The construction is identical to the one used by Krishnamurthy et al~\cite{krishnamurthy2014nonparametric}, but we must make some modifications to ensure that bounded variation condition is satisfied. 
We reproduce the details here for completeness. 

Let $\{\phi_j\}_{j=1}^q$ be an orthonormal collection of functions for $L^2([0,1]^d)$ with $q \ge 4$. 
We can choose $\phi_j$ to satisfy (i) $\phi_1 = 1$, (ii) $\phi_j(x) = 0$ for $x | B(x,\epsilon) \not\subset [0,1]^d$ and (iii) $\|D^r\phi_j\|_{\infty} \le \kappa < \infty$ for all $j$. 
Certainly we can find such an orthonormal system. 

Now for any pair of function $f,g \in L^2([0,1]^d)$, we can find a unit-normed function in $\tilde{w} \in \textrm{span}(\phi_j)$ such that $\tilde{w} \perp \phi_1, \tilde{w} \perp f, \tilde{w} \perp g$. 
If we write $\tilde{w} = \sum_j c_j\phi_j$, we have $D^r\tilde{w} = \sum_j c_i D^r\phi_j$ so that $\|D^r\tilde{w}\|_{\infty} \le \kappa \sum |c_i| \le \kappa \sqrt{q}$ since $\tilde{w}$ is unit normed. 
Thus the vector $w = \tilde{w}/(K\sqrt{q})$ has $\ell_2$ norm equal to $(K\sqrt{q})^{-1}$ while have $\|D^rw\|_{\infty} \le 1$ for all tuples $r$.

For the $u_j$ functions, we use the partition $A_j = \prod_{i=1}^d[j_i m^{-1/d}, (j_i + 1)m^{-1/d}]$ where $j = (j_1, \ldots, j_d)$ and $j_i \in [m^{1/d}]$ for each $i$.
Map $A_j$ to $[0,1]^d$ and appropriately map the densities $p,q$ from $A_j$ to $[0,1]^d$. 
We construct $u_j$ by using the construction for $w$ above on the segment of the density corresponding to $A_j$.
In particular, let $w_j$ be the function from above and let $u_j = w_j(m^{1/d}(x - (j_1, \ldots, j_d)))$.
With this rescaling and shift, $u_j \in A_j$, $\textrm{supp}(u_j) \subset \{x | B(x,\epsilon) \in A_j\}$, and $\int u_j^2(x) = m^{-1} \int w_j^2(x) = \Theta(1/m)$. 
For the last property, by a change of variables and H\'{o}lder's inequality, we have:
\begin{align*}
\|D^ru_j\|_1 = \int |D^rw_j(m^{1/d}(x - (j_1, \ldots, j_d)))| d\mu(x) = \frac{1}{m} \int \|m^{r/d} D^r w_j(y)\| dA_j(y) \le m^{r/d-1}.
\end{align*}
Thus these function $u_j$ meet all of the requirements. 

\section{Proof of Lemma~\ref{lem:var_est}}
Recall that the asymptotic variance of the estimator is:
\begin{align*}
\sigma^2 = 4\left(\Var_{X \sim p}(p(X)) + \Var_{Y \sim q}(q(Y)) + \Var_{X \sim p}(q(X)) + \Var_{Y \sim q}(p(X)\right),
\end{align*}
and our estimator $\hat{\sigma}^2$ is formed by simply plugging in kernel density estimates $\hat{p}, \hat{q}$ for all occurences of the densities.
We will first bound:
\begin{align*}
\EE_{X_1^n,Y_1^n}\left[ | \sigma^2 - \hat{\sigma}^2 |\right]= O(n^{\frac{-\beta}{2\beta+d}}),
\end{align*}
and our high probability bound will follow from Markov's inequality. 
We will show the following bounds, and the expected $\ell_1$ bound will follow by application of the triangle inequality.
Below, let $f,g \in \Wcal_1^\beta(C)$ be any two densities; we will interchangeably substitute $p,q$ for $f,g$. 
\begin{align}
\EE\left[\left| \int \hat{f}^3 - \int f^3 \right| \right] \le O\left(h^{\beta} + \frac{1}{(nh^d)^{1/2}}\right) \label{eq:cubic} \\
\EE\left[\left| \left(\int \hat{f}^2\right)^2 - \left(\int f^2\right)^2\right|\right] \le O\left(h^{2\beta} + \frac{1}{\sqrt{n}} + \frac{1}{nh^{d/2}}\right) \label{eq:squared_squared}\\
\EE\left[\left| \int \hat{f}^2\hat{g} - \int f^2g \right|\right] \le O\left(h^{\beta} + \frac{1}{\sqrt{nh^d}}\right)\label{eq:cross_cubic}\\
\EE\left[\left| \left(\int \hat{f}\hat{g}\right)^2 - \left(\int fg\right)^2 \right|\right] \le O\left(h^{2\beta} + \frac{1}{\sqrt{n}} + \frac{1}{nh^{d/2}}\right) \label{eq:cross_squared}
\end{align}

Before establishing the above inequalities, let us conclude the proof.
The overall rate of convergence in absolute loss is $O(h^\beta + \frac{1}{\sqrt{nh^d}})$. 
TBy choosing $h \asymp n^{\frac{-1}{2\beta+d}}$, the rate of convergence is $O(n^{\frac{-\beta}{2\beta+d}})$.
Finally we wrap up with an application of Markov's Inequality.

Now we turn to establishing the bounds.
For Equation~\ref{eq:cubic}, we can write:
\begin{align*}
\EE\left[\left| \int \hat{f}^3 - \int f^3 \right| \right] &\le \EE \|\hat{f} - f\|_3^3 + 3 \EE\left[\int|f(x) \hat{f}(x)(f(x) - \hat{f}(x))|d\mu(x)\right]\\
& \le \EE\|f - \hat{f}\|_3^3 + 3\EE \|f - \hat{f}\|_{\infty} \|f \hat{f}\|_1\\
& \le O\left(h^{3\beta} + \frac{1}{(nh^d)^{3/2}} + h^{\beta} + \frac{1}{(nh^d)^{1/2}}\right).
\end{align*}
The first step is a fairly straightforward expansion followed by the triangle inequality while in the second step we apply H\"{o}lder's inequality.
The last step follows from well known analysis on the rate of convergence of the kernel density estimator. 

For Equation~\ref{eq:squared_squared} we should actually use the $U$-statistic estimator for $\theta_p$ that we have been analyzing all along.
The bound above follows from Theorem~\ref{thm:quad_rate} and the following chain of inequalities:
\begin{align*}
\EE\left[\left| \left(\int \hat{f}^2\right)^2 - \left(\int f^2\right)^2\right|\right] &\le \EE\left[ \left(\int \hat{f}^2 - f^2\right)^2\right] + 2\|f\|_2^2 \EE\left[\left|\int \hat{f}^2 - f^2\right|\right]\\
& \le \EE\left[ \left(\int \hat{f}^2 - f^2\right)^2\right] + C \sqrt{\EE\left[\left(\int \hat{f}^2 - f^2\right)^2\right]}\\
& \le O\left(h^{4\beta} + \frac{1}{n} + \frac{1}{n^2h^d} + h^{2\beta} + \frac{1}{\sqrt{n}} + \frac{1}{nh^{d/2}}\right).
\end{align*}
The first inequality is a result of some simple manipulations followed by the triangle inequality and the second step is Jensen's inequality.
We already have a bound on the MSE of the estimator $\hat{\theta}_p - \theta_p$ which gives us the inequality in Equation~\ref{eq:squared_squared}.
Applying that bound leads to the last inequality.

The bound for Equation~\ref{eq:cross_squared} follows from exactly the same argument with an application Theorem~\ref{thm:bilinear_rate} instead of Theorem~\ref{thm:quad_rate} in the last step.
So we simply need to establish Equation~\ref{eq:cross_cubic}.
\begin{align*}
\EE\left[\left| \int \hat{f}^2\hat{g} - \int f^2g \right|\right] &= \EE\left[\left|\int (\hat{f}^2 - f^2)\hat{g}\right|\right] + \EE\left[\left|\int f^2 (\hat{g} - g)\right|\right]\\
& \le \EE \|\hat{f}^2 - f^2\|_{2} \|\hat{g}\|_2 + \|f^2\|_2\|\hat{g} - g\|_2\\
& \le \EE \|\hat{f}^2 -f^2\|_2 (\|\hat{g} - g\|_2 + \|g\|_2) + \|f^2\|_2\|\hat{g} - g\|_2\\
& \le O\left(h^{2\beta} + \frac{1}{nh^{d/2}} + \frac{1}{\sqrt{n}} + h^{\beta} + \frac{1}{\sqrt{nh^d}}\right).
\end{align*}
Here we use that $\|\hat{g}\|_1 = 1$ and that $\|f^2\|_2$ and $\|g\|_2$ are both bounded.
We use the standard rate of convergence analysis of the kernel density estimator to bound $\EE \|\hat{g} - g\|_2 \le O(h^\beta + (nh^d)^{-1})$.
We finally use Theorem~\ref{thm:quad_rate} to bound $\|\hat{f}^2- f^2\|_2$.
Note that we are exploiting independence between the samples for $\hat{f}$ and $\hat{g}$ to push the expectation inside of the product in the first term.
In the last line we omitted the term $\EE \|\hat{f}^2 - f^2\|_2\|\hat{g} - g\|_2$ since it converges much faster than the other two terms.

To prove the second bound, we show that $\bar{\sigma}^2$ is close to $\sigma^2$.
We just have to look at two forms:
\begin{align*}
T_1 = \int \bar{p}^2(x)p(x) - \int p^3(x) \qquad T_2 = \left(\int \bar{p}(x)p(x)\right)^2 - \left(\int p^2(x)\right)^2.
\end{align*}
For $T_1$ we can write:
\begin{align*}
T_1 &= \int (\bar{p}^2(x) - p^2(x)) p(x) = \int (\bar{p}(x) - p(x))(\bar{p}(x) - p(x) + 2p(x)) p(x)\\
& = \int (\bar{p}(x) - p(x))^2 p(x) + 2 \int p^2(x)(\bar{p}(x) - p(x))\\
& \le \left(\sup_x|\bar{p}(x)  - p(x)|\right)^2 + 2 \|p\|_2^2 \sup_x |\bar{p}(x) - p(x)| \le O(h^{2\beta} + h^\beta),
\end{align*}
since $p$ is $L_2$-integrable and the kernel density estimator has point-wise bias $O(h^\beta)$.

For $T_2$ we have:
\begin{align*}
T_2 &= \left(\int (\bar{p}(x) - p(x))p(x)\right)^2 + 2\left(\int p^2(x)\right)^2\left(\int (\bar{p}(x) - p(x))p(x)\right)\\
& \le \left(\sup_x |\bar{p}(x) - p(x)|\right)^2 + 2\|p\|_2^4 \sup_x\|\bar{p}(x) - p(x)\| \le O(h^{2\beta} + h^\beta).
\end{align*}

Wwith $h \asymp n^{\frac{-1}{2\beta+d}}$ the additional bias incurred is:
\begin{align*}
\EE \left|\hat{\sigma}^2 - \bar{\sigma}^2\right| &\le \EE \left|\hat{\sigma}^2 - \sigma^2\right| + \left|\sigma^2 - \bar{\sigma}^2\right|
\le O(n^{\frac{-\beta}{2\beta+d}}).
\end{align*}
and so $\hat{\sigma}^2$ is an equally good estimator of $\sigma^2$ and $\bar{\sigma}^2$ (up to constants). 

\section{A Convolution Lemma}
In this section we show that bounded-variation smoothness is additive under convolution. 
\begin{lemma}
If $f,g  \in \Wcal_1^\beta(\RR^d, C)$, 
then $h = f\star g \in \Wcal_1^{2\beta}(\RR^d, C^2)$. 
\label{lem:convolution_smoothness}
\end{lemma}
\begin{proof}
The proof uses the fact that:
\[
\frac{\partial h(x)}{\partial x} = \left(\frac{\partial f}{\partial x} \star g\right)(x)
\]
which follows by pushing the derivative operator inside of the integral and continuity of $f, g$ and their derivatives.
Using the above identity, we have:
\[
\frac{\partial^{2\beta}h(x)}{\partial x^{2\beta}} = \left(\frac{\partial^\beta f}{\partial x^\beta} \star \frac{\partial^\beta g}{\partial x^\beta}\right)(x),
\]
or more concisely:
\[
\|h^{(2\beta)}\|_1 = \|f^{(\beta)}\star g^{(\beta)}\|_1 \le \|f^{(\beta)}\|_1 \|g^{(\beta)}\|_1 \le C^2.
\]
The first inequality is Young's inequality. This implies that $L_1$ is closed under convolution. 

It is clear, by the fact that derivatives can be distributed across the convolution that for $k < 2\beta$, $D^kh \in L^1$.
This proof strategy extends mutatis mutandis to higher dimension. 
\end{proof}

\end{document}